\documentclass[twoside,11pt]{article}
\usepackage[abbrvbib, preprint]{jmlr2e}

\usepackage[normalem]{ulem}
\usepackage{soul}

\usepackage{graphicx, color, caption, enumitem}
\usepackage{wrapfig}

\usepackage{mathtools}
\usepackage[textsize=tiny]{todonotes}

\usepackage{dsfont}
\usepackage{shadethm}
\usepackage{parskip}

\usepackage{enumerate, color, xcolor}
\usepackage{graphicx}
\usepackage{subfigure}
\usepackage{url}
\usepackage{multirow}
\usepackage{hhline}
\usepackage{theoremref}
\usepackage{bbm}

\newcommand{\iu}{\mathrm{i}\mkern1mu}

\definecolor{darkred}{rgb}{0.7,0,0}
\definecolor{darkblue}{rgb}{0,0,0.7}
\definecolor{darkgreen}{rgb}{0,0.4,0}
\hypersetup{
	linktocpage={true},
	linkcolor = {darkblue},
	citecolor = {darkgreen} %
}

\begin{document}
\title{Privacy of SGD under Gaussian or Heavy-Tailed Noise: Guarantees without Gradient Clipping}

\author{%
 \name Umut \c{S}im\c{s}ekli \email umut.simsekli@inria.fr \\
 \addr Inria, CNRS, Ecole Normale Sup\'{e}rieure\\ 
PSL Research University, Paris, France.\\
\AND
 \name Mert G\"{u}rb\"{u}zbalaban \email mg1366@rutgers.edu\\
 \addr Department of Management Science
 and Information Systems\\ Rutgers Business School, Piscataway, NJ, USA \\ %
 \AND
 \name Sinan Y{\i}ld{\i}r{\i}m \email sinanyildirim@sabanciuniv.edu\\
 \addr Faculty of Engineering and Natural Sciences\\Sabanc{\i} University, Istanbul, Turkey.\\
\AND
\name Lingjiong Zhu  \email zhu@math.fsu.edu \\
 \addr Department of Mathematics\\ Florida State University, Tallahassee, FL, USA.
}

\maketitle

\begin{abstract}

The injection of heavy-tailed noise into the iterates of stochastic gradient descent (SGD) has garnered growing interest in recent years due to its theoretical and empirical benefits for optimization and generalization. However, its implications for privacy preservation remain largely unexplored. Aiming to bridge this gap, we provide differential privacy (DP) guarantees for noisy SGD, when the injected noise follows an $\alpha$-stable distribution, which includes a spectrum of heavy-tailed distributions (with infinite variance) as well as the light-tailed Gaussian distribution. Considering the $(\epsilon, \delta)$-DP framework, we show that SGD with heavy-tailed perturbations achieves $(0,  {\mathcal{O}}(1/n))$-DP for a broad class of loss functions which can be non-convex, where $n$ is the number of data points. As a remarkable byproduct, contrary to prior work that necessitates bounded sensitivity for the gradients or clipping the iterates, our theory can handle \emph{unbounded gradients without clipping}, and reveals that under mild assumptions, such a projection step is not actually necessary. 
Our results suggest that, given other benefits of heavy-tails in optimization, heavy-tailed noising schemes can be a viable alternative to their light-tailed counterparts.
\end{abstract}


\section{Introduction}

\paragraph{Context.} Most machine learning problems can be represented in an \emph{empirical risk minimization} (ERM) framework, where the goal is to minimize a loss function in the following form:
\begin{align} \label{eq: empirical risk minimization}
    \min_{\theta \in \mathbb{R}^d} \left\{F(\theta,X_n) := \frac1{n}\sum_{x \in X_n} f(\theta, x) \right\}.
\end{align}
Here, $X_n := \{x_1, \dots, x_n\} \in \mathcal{X}^n$ is a dataset with $n$ data points that are assumed to be independent and identically distributed (i.i.d.) from an underlying data distribution, $f$ is the loss incurred by a single data point, and $\theta$ is the parameter vector.

We will consider \emph{noisy} stochastic gradient descent (SGD) to solve \eqref{eq: empirical risk minimization} that is based on the following recursion:
\begin{equation} 
\theta_{k}=\theta_{k-1}-\eta{\nabla}F_{k}(\theta_{k-1},X_{n}) + \sigma \xi_{k},
\qquad
{\nabla}F_{k}(\theta, X_{n}):=\frac{1}{b}\sum_{i\in\Omega_{k}}\nabla f(\theta,x_{i}), 
\label{eqn:sgd}
\end{equation}
where $\eta>0$ is the step-size, $\Omega_{k}$ is a random subset of $\{1,2,\ldots,n\}$ with the batch-size $b$, independently and uniformly sampled at the $k$-th iteration, and $(\xi_k)_{k \geq 1}$ is a sequence of independent and identically distributed (i.i.d.) noise vectors. 
This algorithmic framework generalizes several practical settings, the most well known being stochastic gradient Langevin dynamics (SGLD) \citep{welling2011bayesian}, which is obtained when $\xi_k$ is Gaussian distributed.

Understanding the differential privacy (DP) of this algorithm is essential, as it ensures rigorous protection of sensitive data, enables safe deployment in privacy-critical applications, and informs algorithmic choices such as noise distribution, step-size, and projection mechanisms.  The DP framework \citep{dwork2014algorithmic} concerns designing randomized algorithms that aim at producing random outputs that still carry inferential utility while providing statistical deniability about the input dataset.
Noisy SGD with Gaussian and Laplace noise distributions have been studied extensively for their DP guarantees in the literature (see, e.g., \citet{chaudhuri2011differentially, abadi2016, Wang_et_al_2017, yu2019differentially, Kuru_et_al_2022} among many). In the Gaussian noise case, the privacy properties have been analyzed by using different tools \citep{ganesh2020faster,altschuler2022privacy,chourasia2021differential, ye2022differentially, ryffel2022differential}, which mainly cover convex and strongly convex $f$ and require bounded gradients $\nabla f$. The bounded gradient assumption often further necessitates the recursion \eqref{eqn:sgd} to be appended with a projection step onto a bounded set at every iteration. Recently, \citet{asoodeh2023privacy,murata2023diff2,chien2024langevin} provided differential privacy guarantees for noisy SGD under non-convex losses as well; however, they still require a projection step. Moreover, the widespread use of gradient clipping, a common workaround to enforce boundedness, can distort gradient information and negatively impact optimization dynamics, especially in non-convex settings \citep{chen2020understanding}.

Emerged in a non-DP related context, there has been an increasing interest in injecting \emph{heavy-tailed} noise to the SGD iterates, potentially with unbounded higher-order
moments, i.e., $\mathbb{E}[\|\xi_k\|^\alpha] = + \infty$ for some $\alpha > 1$. 
Despite the `daunting' connotation of heavy tails, it has been shown that using heavy-tailed noise in stochastic optimization can be surprisingly beneficial. In the context of learning theory, \citet{simsekli2020hausdorff,barsbey2021heavy,raj2022algorithmic,lim2022chaotic,raj2023algorithmic} showed that using a heavy-tailed noise can result in better generalization performance. More precisely, they investigated the generalization error of \eqref{eqn:sgd} i.e., $|\mathbb{E}_{X_n}[F(\theta,X_n)] - F(\theta,X_n)| $, when $\xi_k$ is heavy-tailed with infinite variance. They demonstrated that the generalization error decreases as $\xi_k$ becomes increasingly heavy-tailed, up to a certain threshold.
In another recent study, \citet{wan2023implicit} proved that the combination of heavy-tailed noise and overparametrization in a one-hidden-layer neural network setting forces the network weights $\theta_k$ to be `compressible', which is a crucial property in low-resource settings. Hence, using heavy-tailed noise can bring about important advantages from computational and generalization error perspectives.

On the other hand, heavy-tailed noise can introduce challenges in terms of minimizing the empirical risk $\theta\mapsto F(\theta,X_n)$ due to the potentially large fluctuations in $\theta_k$ \citep{csimcsekli2020fractional,gorbunov2020stochastic}. Nevertheless, it has been shown that if the loss function $f$ is sufficiently regular, SGD with heavy tails (without gradient- or iterate-clipping) still converges. In a non-convex setting, \citet{csimcsekli2019heavy} showed that $\min_{1\leq k\leq K} \mathbb{E}\|\nabla F(\theta_k, X_n)\|$ can be made arbitrarily small, and \citet{wan2023implicit} showed that one-hidden-layer neural networks with smooth activation functions indeed fall into the setting of \citet{csimcsekli2019heavy}, hence heavy-tailed SGD converges near a critical point. In the case where $f$ belongs to a class of strongly convex functions, \citet{wang2021convergence} showed that SGD with decaying step-sizes converges to the unique global minimum $\theta_\star$ in $L^p$, namely $\mathbb{E}\|\theta_k -\theta_\star\|^p \to 0$ as $k \to \infty$, if $\mathbb{E}\|\xi_k\|^p < \infty$ for some $1\leq p <\alpha$ (recall that $\mathbb{E}\|\xi_k\|^\alpha = +\infty$).

Even though heavy-tailed SGD has been analyzed from learning theoretical and optimization theoretical perspectives, it is still not clear what the effect of injecting heavy-tailed noise would be in terms of DP.
In a broader context, several recent works have explored the intersection of DP and heavy-tailed noise, each from a distinct perspective. In a control-theoretic context, \citet{ito2021privacy} examined privacy in linear dynamical systems using $\alpha$-stable distributed noise to better mask outliers. 
Very recently, \cite{zawacki2024symmetric,zawacki2025heavy} introduced the `symmetric $alpha$-stable mechanism', which achieves pure differential privacy—a stronger guarantee than the approximate DP typically attained by Gaussian mechanisms—while retaining desirable properties such as closure under convolution. Their work primarily addresses static query settings and does not consider optimization dynamics or SGD. 
Finally, perhaps closest to our study, \cite{asi2024private} investigated private stochastic convex optimization under heavy-tailed gradient assumptions, with finite second-order and bounded gradient assumptions. They derived near-optimal rates using clipped DP-SGD. While they also focused on statistical risk minimization, their approach required convex losses, finite variance, and more importantly clipping, which—as our work emphasizes—can obscure the beneficial effects of heavy-tailed noise.
In contrast to all the prior art, our work focuses specifically on analyzing the privacy guarantees of non-clipped noisy SGD with heavy-tailed noise with infinite variance in non-convex settings, providing insights that are both algorithmically and analytically distinct.

\paragraph{Objective.} In this study, we will provide DP guarantees for noisy SGD, \emph{without} gradient clipping, when the noise follows a class of heavy-tailed distributions.  We will deliberately focus on the non-clipped case, as clipping effectively suppresses the heavy-tailed behavior, thereby eliminating all the potential benefits associated with it. 
We will choose the class of the noise distribution as the $\alpha$-stable distribution family: a class of distributions that include both the Gaussian distribution and a wide range of heavy-tailed distributions with diverging high-order moments (we will provide more details in Section~\ref{sec:stable}).

Drawing inspirations from a recent study on algorithmic stability \citep{zhu2023uniform}, we take an alternative route through Markov chain stability theory and develop a novel analysis technique for understanding the privacy properties of noisy SGD.  The analysis involves a direct approach where, for an arbitrary $X_{n}$,
we consider (theoretically) running SGD on a `neighboring' data set $\hat{X}_{n}:=\{\hat{x}_1, \ldots, \hat{x}_n\} = \{x_1, \ldots, x_{i-1}, \hat{x}_i , x_{i+1},\ldots {x}_n\} \in \mathcal{X}^n$ that differs from $X_{n}$ by at most one element, i.e., 
\begin{equation}
\hat{\theta}_{k}=\hat{\theta}_{k-1}-\eta{\nabla}F_{k}(\hat{\theta}_{k-1},\hat{X}_{n}) + \sigma \xi_k,
\qquad
{\nabla}F_{k}({\theta},\hat{X}_{n})=\frac{1}{b}\sum_{i\in\Omega_{k}}\nabla f\left({\theta},\hat{x}_{i}\right),
\label{eqn:sgd2}
\end{equation} 
and analyze the probabilistic difference between its iterates and those in \eqref{eqn:sgd}. If the distributions of $\theta_k$ and $\hat{\theta}_k$ are close in some sense, we can conclude that changing one data point in the dataset would not have a significant impact, hence the privacy of an individual data point can be preserved.  By making use of relatively recent results from the theory of Markov chain stability \citep{RS2018}, we estimate the \emph{total variation} (TV) distance between the laws of $\theta_k$ and $\hat{\theta}_k$, which can be immediately turned into bounds on the privacy leakage. 

\paragraph{Contributions.} Our contributions are as follows:
\begin{itemize}%
\item By building up on the $(\epsilon, \delta)$-DP framework \citep{Dwork_2006, dwork2014algorithmic} (to be introduced formally in the next section), we show that for $\alpha>1$ and \emph{dissipative} (potentially non-convex) loss functions, noisy SGD with $\alpha$-stable perturbations achieves $(0,\delta)$-DP with $\delta = {\mathcal{O}}(1/n)$, where $n$ is the number of data points.  

\item A remarkable outcome revealed by our theory is that the bounded gradient assumption as well as the projection step appended to SGD are \emph{not} actually required for obtaining DP (including the Gaussian case). Our theory shows that SGD enjoys DP without needing projections once the gradients satisfy a pseudo-Lipschitz continuity condition (which has already been considered in the literature and holds for practical problems such as linear and logistic regression) and assuming the data is bounded with high probability (e.g., sub-Gaussian data).

\item  In terms of the dimension $d$, our bound has a dependency of order $d^{\frac{\alpha +1}{2}}$ for large $d$, where $\alpha$ is the parameter that determines %
the heaviness of the tails. This result suggests a potential benefit of heavy tails: as the noise gets heavier tailed (smaller $\alpha$), the dependency of the bound on $d$ weakens.
\item As another technical contribution, rather than directly upper bounding the total variation (TV) distance—which proves to be more challenging—we instead upper bound a related quantity known as the $V$-norm, which itself upper bounds the TV. This approach requires the design of suitable Lyapunov functions and the demonstration of their contraction properties with respect to the Markov kernels induced by noisy SGD. To achieve this goal, we make novel connections to stochastic analysis.

\item Similar to its Gaussian counterparts \citep{chourasia2021differential,ryffel2022differential, chien2024langevin}, our bounds are time-uniform, i.e., they do not increase with the increasing number of iterations. 
\end{itemize}
Besides being able to handle both heavy-tailed and Gaussian noising schemes, allowing for non-convexity, and not requiring projections, our rates match the prior art.
Perhaps surprisingly, this observation reveals that the heavy-tailed noising mechanism in SGD provides similar DP guarantees compared to the Gaussian case, where the dimension dependency gets improved as the tails get heavier. Our results suggest that the considered heavy-tailed mechanism is a viable alternative to its light-tailed counterparts. Moreover, the connection we establish between Markov chain stability theory and DP offers a new analytical perspective that we believe can inform and inspire future developments in the field.

\section{Technical Background}

In this section, we will define the basic technical notions that will be essential for our study. 

\subsection{Differential privacy and the TV distances}
DP is a property that can be attached to randomized algorithms. A randomized algorithm takes a dataset as input and returns a random variable as output, where the source of randomness is in the algorithm's inner mechanism. We give a formal definition below.
\begin{definition}[$(\epsilon,\delta)$-DP, \citet{dwork2014algorithmic}]
    Let $\epsilon, \delta \geq 0$. A randomized algorithm $\mathcal{A}$ is called $(\epsilon,\delta)$-differentially private, if for all neighboring datasets $X, \hat{X} \in \mathcal{X}^n$   that differ by one element (denoted by $X\cong \hat{X}$), and for every measurable $E \subset \mathrm{Range}(\mathcal{A})$, the following relation holds:
    \begin{align}
        \label{eqn:dp}
        \mathbb{P}\left(\mathcal{A}(X) \in E \right) \leq \exp(\epsilon) \mathbb{P}\left(\mathcal{A}(\hat{X}) \in E \right) + \delta.
    \end{align}
\end{definition}

Later, we will exploit a relation between DP and TV distance, whose formal definition is given as follows.
\begin{definition}[TV distance]
    Let $\mu,\nu$ be two probability distributions defined on the same measurable space $(\Omega, \mathcal{F})$. The $\mathrm{TV}$ distance between $\mu$ and $\nu$ is defined as follows:
    \begin{align}
        \label{eqn:tv}
        \mathrm{TV}(\mu,\nu) :=  \sup_{E \in \mathcal{F}} |\mu(E) - \nu(E) | = \frac1{2} \sup _{|f| \leq 1}\left|\int_{\Omega} f(y)\left( \mu(\mathrm{d} y)- \nu(\mathrm{d} y)\right)\right|.
    \end{align}
\end{definition}
With a slight abuse of notation, for two random variables $X, Y$, we will denote  
\begin{align*}
    \mathrm{TV}(X,Y) := \mathrm{TV}( \mathrm{Law}(X), \mathrm{Law}(Y)).
\end{align*}

The following result establishes the link between TV stability and DP.
\begin{proposition}
\label{prop:tv_dp}
    Let $\mathcal{A}$  be a randomized algorithm and $\delta \geq 0$. Then, the following stability condition holds for $\mathcal{A}$:
    \begin{align}
        \mathrm{TV}(\mathcal{A}(X), \mathcal{A}(\hat{X})) \leq \delta \qquad \text{for any} \quad X \cong \hat{X}
    \end{align}
    if and only if $\mathcal{A}$ is $(0,\delta)$-DP. 
\end{proposition}
\begin{proof}
The result directly follows from the definitions of the $\mathrm{TV}$-distance and $(0,\delta)$-DP. 
\end{proof}
Similar links between DP and TV have been already considered in \citet{cuff2016differential,kalavasis2023statistical}.

\subsection{Markov chain stability}
\label{sec:stab}

In this paper, our goal will be to upper bound $\mathrm{TV}(\theta_k, \hat{\theta}_k)$, as this would immediately give as a DP guarantee, thanks to Proposition~\ref{prop:tv_dp}. To this end, we will resort to the Markov chain stability theory which was 
 developed by \citet{RS2018}.

Let $(\theta_{k})_{k\geq0}$ be a Markov chain in $\mathbb{R}^d$ with transition kernel $P$
and initial distribution $p_{0}$, i.e., for any measurable set $A\subseteq\mathbb{R}^{d}$, 
$\mathbb{P}(\theta_{k}\in A|\theta_{0},\cdots,\theta_{k-1})
=\mathbb{P}(\theta_{k}\in A|\theta_{k-1})=P(\theta_{k-1},A)$,
and $p_{0}(A)=\mathbb{P}(\theta_{0}\in A)$ 
and $k\in\mathbb{N}$. 
Let $(\hat{\theta}_{k})_{k\geq0}$ be another Markov
chain with transition kernel $\hat{P}$ and initial distribution $\hat{p}_{0}$.
We denote by $p_{k}$ the distribution of $\theta_{k}$ and by $\hat{p}_{k}$
the distribution of $\hat{\theta}_{k}$. In this context, \citet{RS2018} developed generic analysis tools for estimating $\mathrm{TV}(\theta_k, \hat{\theta}_k) = \mathrm{TV}(p_k, \hat{p}_k)$ by using the properties of the transition kernels associated with each chain. Before proceeding to their result, we first need to define the notion of the $V$-norm of a signed measure and the $V$-uniform ergodicity for Markov chains.

\begin{definition}[$V$-norm]
    Let $\mu$ and $\nu$ be two probability distributions on $\mathbb{R}^d$ and $V: \mathbb{R}^d \rightarrow[1, \infty]$ be measurable function with finite moments with respect to $\mu$ and $\nu$. Then, the $V$-norm between $\mu$ and $\nu$ is defined as follows:
    \begin{align*}
        \left\|\mu-\nu\right\|_V
:=\sup _{|f| \leq V}\left|\int_{\mathbb{R}^d} f(y)\left( \mu(\mathrm{d} y)- \nu(\mathrm{d} y)\right)\right|.
    \end{align*}
\end{definition}
One can view the $V$-norm as a generalization of the total variation norm: the $V$-norm reduces to the total variation norm when $V \equiv 1$ \citep[Section D.3]{douc2018markov}. On the other hand, for a general $V \geq 1$, we define the following notion of ergodicity, which will be central in our analysis.
\begin{definition}[$V$-uniform ergodicity]
\label{def:erg}
    A Markov process $(\theta_{k})_{k\geq0}$ with the transition kernel $P$ is called $V$-uniformly ergodic with an invariant distribution $\pi$, if there exists a $\pi$-almost everywhere finite measurable function $V: \mathbb{R}^d \rightarrow[1, \infty]$ with finite moments with respect to $\pi$ and there are constants $\rho \in[0,1)$ and $C \in(0, \infty)$ such that
\begin{align*}
\left\|P^k(\theta, \cdot)-\pi\right\|_V
=\sup _{|f| \leq V}\left|\int_{\mathbb{R}^d} f(y)\left(P^k(\theta, \mathrm{d} y)-\pi(\mathrm{d} y)\right)\right| \leq C V(\theta) \rho^k, 
\end{align*}
for any $\theta \in \mathbb{R}^d$ and $k \in \mathbb{N}$.
Thus, it holds that
$$
\sup _{\theta \in \mathbb{R}^d} \frac{\left\|P^k(\theta, \cdot)-\pi\right\|_V}{V(\theta)} \leq C \rho^k.
$$
\end{definition}
This notion has been widely used in the analysis of Markov processes \citep{Meyn1993}. By assuming that $(\theta_k)_{k\geq 0}$ is ergodic in the sense of Definition~\ref{def:erg}, we have the following estimate on the $V$-norm.
 
\begin{lemma}[{\citet[Corollary 3.3]{RS2018}}]
\label{lem:perturb_new}
Let $P$ be $V$-uniformly ergodic, that is, there are constants $\rho \in[0,1)$ and $C \in$ $(0, \infty)$ such that
\begin{align}
    \left\|P^k(\theta, \cdot)-\pi\right\|_V \leq C V(\theta) \rho^k, \quad \theta \in \mathbb{R}^d, k \in \mathbb{N} .
    \label{eqn:perturb_erg}
\end{align}
We also assume that there are numbers $\beta \in(0,1)$ and $H \in(0, \infty)$ and a measurable Lyapunov function $\hat{V}: \mathbb{R}^d \rightarrow[1, \infty)$ of $\hat{P}$ such that
\begin{equation}
\label{eqn:fosterlyap}
(\hat{P} \hat{V})(\theta) \leq \beta \hat{V}(\theta)+H.
\end{equation}
Let
\begin{align}
    \gamma=\sup _{\theta \in \mathbb{R}^d} \frac{\|P(\theta, \cdot)-\hat{P}(\theta, \cdot)\|_V}{\hat{V}(\theta)} \quad \text { and } \quad \kappa=\max \left\{\hat{p}_0(\hat{V}), \frac{H}{1-\beta}\right\},
    \label{eqn:pert_gamma}
\end{align}
with $\hat{p}_0(\hat{V})=\int_{\mathbb{R}^d} \hat{V}(\theta) \mathrm{d} \hat{p}_0(\theta)$. Then,
$$
\left\|p_k-\hat{p}_k\right\|_V \leq C\left(\rho^k\left\|p_0-\hat{p}_0\right\|_V+\left(1-\rho^k\right) \frac{\gamma \kappa}{1-\rho}\right) .
$$
    
\end{lemma}

While it might seem technical, this result will prove very useful for developing DP bounds for noisy SGD. Informally, Lemma~\ref{lem:perturb_new} suggests a three-step recipe for bounding the TV distance between $w_k$ and $\hat{w}_k$: (i) identify a Lyapunov function $V$ and show that $(w_k)_{k\geq 0}$ is $V$-uniformly ergodic, (ii) identify another Lyapunov function $\hat{V}$, estimate the constants in \eqref{eqn:fosterlyap}, and (iii) bound the $V$-norm of the difference between between \emph{one-step} transition kernels $P$ and $\hat{P}$, weighted by $\hat{V}$ (cf.\ \eqref{eqn:pert_gamma}). Once these steps are performed, Lemma~\ref{lem:perturb_new} immediately gives an upper bound on $\mathrm{TV}(w_k, \hat{w}_k)$: we have that 
\begin{align*}
   \mathrm{TV}(w_k, \hat{w}_k) \leq \frac1{2} \left\|p_k-\hat{p}_k\right\|_V\,
\end{align*}
since $V \geq 1$  (cf.\ \eqref{eqn:tv}). This ultimately provides a DP guarantee by Proposition~\ref{prop:tv_dp}.

Note that, in a learning theory context, \citet{raj2023algorithmic} followed a stochastic differential equations-based route for obtaining bounds on the Wasserstein distance between the laws of $\theta_k$ and $\hat{\theta}_k$. Their analysis cannot be directly used in our setting as the Wasserstein distance does not have a direct link with DP. On the other hand, \citet{zhu2023uniform} followed a Markov chain stability route for establishing Wasserstein-stability of noisy SGD; however, similar to \citet{raj2023algorithmic}, their approach does not have a direct link with DP and it also does not apply to the heavy-tailed setting as they require bounded second-order moments.

\subsection{Stable distributions}

\label{sec:stable}
We will consider a specific noise distribution for $(\xi_k)_{k\geq 1 }$, such that we will assume that it follows a \emph{rotationally invariant stable distribution}, which has the following characteristic function for $\alpha \in (0,2]$:
\begin{align}\label{stable:dist}
    \mathbb{E}\left[\exp(\iu u^\top \xi_k)\right]  = \exp(- \|u\|^\alpha),
\end{align}
for all $u \in \mathbb{R}^d$ and $k \geq 1$, where $\iu := \sqrt{-1}$.
Here $\alpha\in(0,2]$ is known as the tail-index
that determines
the tail thickness of the distribution. The tail becomes heavier as $\alpha$ gets smaller.
In particular, when $\alpha=2$, the stable distribution reduces to the Gaussian distribution.
When $0<\alpha<2$, the moments of stable distributions
are finite only up to the order $\alpha$
in the sense that the $p$-th moments are finite
if and only if $p<\alpha$,
which implies infinite variance when $\alpha<2$ and infinite mean when $\alpha\leq 1$.
In the rest of the paper, we focus on the regime $\alpha\in(1,2]$, which
includes the Gaussian case ($\alpha=2$) and the heavy-tailed case ($1<\alpha<2$)
with a finite mean. Similar noise models for SGD have been already considered in prior work, see e.g., \citet{nguyen2019first,csimcsekli2020fractional,wan2023implicit}. 
For further properties of stable distributions, we refer to \citet{ST1994}.

\section{Main Assumptions}

In this section, we state our main assumptions. We first state some regularity conditions for the loss $f$, then we will introduce an \emph{optional}  condition, which is not strictly necessary, yet significantly simplifies the proofs.

\subsection{Regularity conditions}

In this section, we will present the main assumptions that will be used throughout the paper. Our first assumption is a pseudo-Lipschitz continuity assumption on the gradient of the loss function.

\begin{assumption}\label{assump:1}
For every $x \in \mathcal{X}$, $f(\cdot,x)$ is
differentiable
and 
there exist  constants $K_1, K_2>0$ such that
for any $\theta,\hat{\theta}\in\mathbb{R}^{d}$ and every $x, \hat{x} \in \mathcal{X}$, 
\begin{align}
\|\nabla f(\theta, x) - \nabla f(\hat{\theta},\hat{x})\| 
\leq  K_1 \|\theta- \hat{\theta}\| + K_2 \| x-\hat{x} \|( \|\theta \| + \|\hat{\theta}\| +1). \label{eq: regularity assumption - 1}
\end{align} 
\end{assumption}
This assumption has been used for decoupling the data and the parameter and it has been considered in various settings. 
It is similar to the pseudo-Lipschitz-like condition studied by \citet{erdogdu2018global}.
It is satisfied for many various problems such as GLMs \citep{bach2014adaptivity}.

Our second assumption is a uniform dissipativity condition on the loss function.

\begin{assumption}\label{assump:grad}
There exist universal positive constants $B$, $m$, and $K$ such that
for any $\theta_{1},\theta_{2}\in\mathbb{R}^{d}$ and $x\in\mathcal{X}$,
\begin{align*}
\Vert\nabla f(0,x)\Vert\leq B,
\qquad
\left\langle\nabla f(\theta_{1},x)-\nabla f(\theta_{2},x),\theta_{1}-\theta_{2}\right\rangle
\geq
m\Vert\theta_{1}-\theta_{2}\Vert^{2}-K.
\end{align*}
\end{assumption}

This dissipativity assumption is satisfied 
when the loss function admits some gradient growth in radial directions outside a compact set. 
Also, any function that is strongly convex outside of a ball of some positive radius satisfies Assumption~\ref{assump:grad}.
In particular, this assumption is satisfied for some one-hidden-layer neural networks \citep{akiyama2022excess}, 
non-convex formulations of classification problems (e.g.\ in logistic regression with a sigmoid/non-convex link function), 
robust regression problems \citep{gao2022global}, sampling and Bayesian learning problems and global convergence in non-convex optimization problems \citep{raginsky2017non,gao2022global}. Moreover, any regularized regression problem where the loss is a strongly convex quadratic plus a smooth penalty that grows slower than a quadratic satisfies Assumption~\ref{assump:grad}, 
such as smoothed Lasso regression; see \citet{erdogdu2022} for more examples. Informally, the constant $K$ measures the `level of non-convexity' of the problem: when $K=0$ the loss becomes strongly convex; for $K>0$ the function class can start accommodating non-convex functions. 

In Proposition~\ref{lemma-assump-satisfied-ridge-regression} and Proposition~\ref{prop:log_reg} in Appendix~\ref{sec:exp_constants}, we show that our assumptions are satisfied for $\ell_2$-regularized linear and logistic regression problems and we explicitly compute the required constants.

\subsection{(Optional) Existence of a universal stable point}

In this section, we introduce an assumption that requires the existence of a `universal stable point'. This assumption is not required for obtaining our bounds; however, in case it is assumed to hold, we will show that we can obtain tighter results.  

\begin{assumption}
    \label{asmp:interp}
    There exists $\vartheta_{\star} \in \mathbb{R}^d$ such that for every $x \in \mathcal{X}$, $\nabla f (\vartheta_\star, x) = 0$.
\end{assumption}

This condition is similar to the `stable-point interpolation' condition as defined by \citet[Definition 4]{Mishkin_2020} and also to the `interpolation condition' as considered by \citet[Definition 4.9]{garrigos2023handbook}. However, it is milder in the sense that, we do not require the implication $\nabla F(\theta,X_n) = 0 \Rightarrow \nabla f(\theta, x_i) =0 $ for every admissible $\theta$ as opposed to \citet{Mishkin_2020}, nor nor do we impose the constraint that $\vartheta_\star$ has to be a minimizer as it is required in \citet{garrigos2023handbook}. Instead, Assumption~\ref{asmp:interp} requires the \emph{existence} of a single stable point $\vartheta_\star$ such that the gradient of $f$ vanishes at $\vartheta_\star$. However, we need this condition to hold for every $x \in \mathcal{X}$ contrary to \citet{Mishkin_2020} and \citet{garrigos2023handbook}, who require their conditions to hold only on a given training set.   

To illustrate the assumption, we provide the following two examples where the condition holds. 

\textbf{Example 1 (Neural networks).} Consider a supervised learning setting $x = (a,y)$, where $a \in \mathbb{R}^p$ is the feature and $y\in \mathbb{R}$ is the label and consider the following fully-connected neural network architecture: $f(\theta, x) = \ell ( \theta_2^\top h (\theta_1^\top a), y) $, where $\ell$ is a differentiable loss function, $\theta_1 \in \mathbb{R}^{p \times d_1}$, $\theta_2 \in \mathbb{R}^{d_1}$ are the network weights, $\theta \equiv \{\theta_1, \theta_2\}$ and $h : \mathbb{R} \to \mathbb{R}$ is a differentiable activation function applied component-wise satisfying $h(0)=0$.\footnote{The condition $h(0)=0$ is satisfied by many smooth activation functions such as hyperbolic tangent, ELU, SELU, and GELU.} Then Assumption~\ref{asmp:interp} holds with $\vartheta_\star = 0 \in \mathbb{R}^d$. 

\textbf{Example 2 (Realizable settings).} Consider the same supervised learning setting with $x = (a,y)$ and assume that exists a parametric function $g_{\vartheta_\star} \in \{g_\theta : \theta \in \mathbb{R}^d\}$ such that for every $x = (a,y) \in \mathcal{X}$, $y = g_{\vartheta_\star}(a)$ (i.e., no label noise). If we have $f(\theta, x) = \ell(g_\theta(a),y)$ for some nonnegative and differentiable $\ell$ with $\ell(y',y') =0$ for all $y' \in \mathbb{R}$, then Assumption~\ref{asmp:interp} holds with $\vartheta_\star$. %
Note that in this case $f(\vartheta_\star,x) =0$ for all $x\in \mathcal{X}$, which is more than what is required by Assumption~\ref{asmp:interp}. This setting is sometimes called a `well-specified statistical model' \citep{bickel2015mathematical}. 

We shall underline that Assumption~\ref{asmp:interp} is optional and only requires the existence of a universal stable point, we do not need the optimization algorithm to converge towards it.  

\section{Privacy of Noisy Gradient Descent}

\label{sec:gd}

We first focus on the noisy gradient descent (GD) case where $\nabla F_k = \nabla F$
for all $k$. We handle this setting separately as its proofs are relatively simpler and might be more instructive. More precisely, we consider the following recursion 
\begin{align}
    \label{eqn:gd}
    \theta_{k}=\theta_{k-1}-\eta{\nabla}F(\theta_{k-1},X_{n}) + \sigma \xi_{k},
\end{align}
for $\alpha \in (1,2]$ and we will follow the three-step recipe given in Section~\ref{sec:stab}. Here, the recursion for $(\hat{\theta}_k)_{k\geq 0}$ is defined similarly to the one given in \eqref{eqn:sgd2}:
\begin{align}
    \label{eqn:gdhat}
    \hat{\theta}_{k}=\hat{\theta}_{k-1}-\eta{\nabla}F(\hat{\theta}_{k-1},\hat{X}_{n}) + \sigma \xi_{k}.
\end{align}

In the rest of this section, we will follow the `three-step route' suggested by Lemma~\ref{lem:perturb_new} and establish a DP guarantee for noisy GD under $\alpha$-stable noise.

\paragraph{The design of the Lyapunov functions and the distance between one-step transition kernels.} We start by estimating the term $\gamma$ in \eqref{eqn:pert_gamma}, i.e.,
\begin{align}
    \sup _{\theta \in \mathbb{R}^d} \frac{\|P(\theta, \cdot)-\hat{P}(\theta, \cdot)\|_V}{\hat{V}(\theta)}, 
    \label{eqn:gamma_2}
\end{align}
which requires us to define the Lyapunov functions $V$ and $\hat{V}$. This part constitutes the `art' part of our analysis as it requires the design of the `good' $V$, $\hat{V}$, that are effective across all parts of the analysis, namely, \eqref{eqn:perturb_erg}, \eqref{eqn:fosterlyap}, and \eqref{eqn:pert_gamma}.

First of all, as we only need to have an estimate on $\mathrm{TV}(\theta_k, \hat{\theta}_k)$ to obtain a DP guarantee, it is tempting to choose $V \equiv 1$, as  with this choice of $V$, we have that $\frac1{2}\|p_k - \hat{p}_k\|_V = \mathrm{TV}(\theta_k, \hat{\theta}_k)$. Hence, if this option was viable, Lemma~\ref{lem:perturb_new} would immediately give us a bound on the $\mathrm{TV}$ distance, which would be sufficient for our purposes. However, under this choice, the $V$-uniform ergodicity condition \eqref{eqn:perturb_erg} would reduce to uniform ergodicity \citep{douc2018markov}, which could only hold when $\theta$ is restricted in a compact space. Since we aim to avoid the use of clipping, we adopt a more refined approach that does not require the iterates to be confined to a bounded domain.

Let us introduce the following (family of) Lyapunov functions for $0 < p < \min(\frac1{2},\alpha-1)$ and $\alpha \in (1,2)$, which we established after a series of trial and error:
\begin{align}
    \label{eqn:v_vhat_new}
    V_p(\theta) :=& \left(1 + \| \theta - \vartheta_\star \|^2\right)^{p/2}, \\
    \label{eqn:v_vhat_new2}
    \hat{V}_p(\theta) :=& V_{1+p}(\theta) = \left(1 + \|\theta - \vartheta_\star\|^2\right)^{(1+p)/2}.
\end{align}
The reasoning of this choice is as follows. In the proof of the following lemma, we show that, with our choice of $V_p$, we have the following estimate:
\begin{align*}
    \|P(\theta, \cdot)-\hat{P}(\theta, \cdot)\|_{V_p} = \mathcal{O} \left(  \frac{\|x - \hat{x}\|  \|\theta\|^{1+p}
 }{n}  \right)   , 
\end{align*}
 where $x,\hat{x}$ are two data points.
Noticing that  $\hat{V}_p(\theta) \approx \|\theta\|^{1+p}$, the terms depending on $\theta$ in \eqref{eqn:gamma_2} gracefully cancel, hence the reason behind our choice of $\hat{V}_p$. Furthermore, $\|x - \hat{x}\|$ is bounded once the data domain is bounded, hence the supremum in \eqref{eqn:gamma_2} is upper-bounded, ultimately circumventing the requirement of gradient clipping, even under heavy tails. The next lemma formalizes this informal explanation.
\begin{lemma}
\label{lem:gamma_gd_new}
      Let $\alpha \in (1,2)$, $p\in (0,\min(\frac1{2},\alpha-1))$, $P$ be the transition kernel associated with the Markov process $(\theta_k)_{k\geq 0}$ in \eqref{eqn:gd} and $\hat{P}$ be the transition kernel associated with $(\hat{\theta}_k)_{k\geq 0}$. Suppose that Assumptions~\ref{assump:1} and \ref{assump:grad} hold and further assume that $\sup_{x,\hat{x} \in \mathcal{X}} \|x-\hat{x}\|\leq D$, for some $D<\infty$. Suppose that Assumption~\ref{asmp:interp} holds and set $V_p(\theta)$ and $  \hat{V}_p(\theta)$ as in \eqref{eqn:v_vhat_new} and \eqref{eqn:v_vhat_new2}, respectively.
      Then, the following inequality holds:
        \begin{align*}
        \gamma =& \sup_{\theta \in \mathbb{R}^d} \frac{\|P(\theta, \cdot) - \hat{P}(\theta, \cdot)\|_{V_p}}{\hat{V}_p(\theta)} 
        \leq \frac{ C_\gamma } {n } ,
    \end{align*}
    where $C_\gamma>0$ is a constant whose explicit expression is provided in the proof.
\end{lemma}
This lemma shows that the transition kernels $P$ and $\hat{P}$ will get closer with the rate of $1/n$ as the number of data points $n$ increases. We note that we place the bounded data assumption in Lemma~\ref{lem:gamma_gd_new} for notational clarity. This condition can be replaced by more general sub-Gaussian (or related) data assumptions, where in that case our bounds would hold in high probability over the data samples.

\paragraph{Estimation of the Lyapunov functions and ergodicity of the Markov chains.}

As the second step, we show that, under our choice of $V_p$ and $\hat{V}_p$, the $V$-uniform ergodicity \eqref{eqn:perturb_erg} and the Lyapunov condition \eqref{eqn:fosterlyap} hold. We start by the ergodicity condition, whose proof is simpler.

\begin{lemma}
\label{lem:erg}

Let $P$ be the transition kernel associated with the Markov process $(\theta_k)_{k\geq 0}$ (i.e., \eqref{eqn:gd}). Suppose that Assumptions~\ref{assump:1} and \ref{assump:grad} hold, $\eta < \min\{m/K_1^2, 1/m\}$. 
Further assume that Assumption~\ref{asmp:interp} holds and set $V_p(\theta) := (1 + \|\theta - \vartheta_\star\|^2)^{p/2}$, where $\vartheta_\star$ is defined in Assumption~\ref{asmp:interp} and $p\in (0, 1]$.
Then, the process \eqref{eqn:gd} admits a unique invariant measure $\pi$ such that the following inequality holds for some constants $c >0$, $\rho \in (0,1)$:
    \begin{align*}
    \left\|P^k(\theta, \cdot)-\pi\right\|_{V_p} \leq c V_p(\theta) \rho^k, \quad \theta \in \mathbb{R}^d, k \in \mathbb{N} .   
    \end{align*}
\end{lemma}
This result shows that $(\theta_k)_{k\geq 0}$ is $V_p$-uniformly ergodic even when the loss can be non-convex.

We now proceed to the final step, where we show that $\hat{V}_p = V_{1+p}$ (for $p<1$) satisfies the Lyapunov condition \eqref{eqn:fosterlyap}. For simplicity, we prove the condition for $V_p$ for $p \geq 1$, which is equivalent to $\hat{V}_p$ for $p<1$. 
\begin{lemma}
\label{lem:erg_new_p2}
Let $\alpha \in (1,2)$, $P$ be the transition kernel associated with the Markov process $(\theta_k)_{k\geq 0}$ (i.e., \eqref{eqn:gd}) and $\hat{P}$ be the transition kernel associated with $(\hat{\theta}_k)_{k\geq 0}$. Suppose that Assumptions~\ref{assump:1} and \ref{assump:grad} hold, and the step-size is chosen as $\eta < \min\{m/(2K_1^2), 1/m,1\}$. 
Suppose that Assumption~\ref{asmp:interp} holds and set $V_p(\theta) := (1 + \|\theta - \vartheta_\star\|^2)^{p/2}$, where $\vartheta_\star$ is defined in Assumption~\ref{asmp:interp} and $p\in [1, \alpha)$.
   Then, the following inequalities hold:
   \begin{align*}
            (PV_p)(\theta)  \leq \beta_p  V_p(\theta) + H_p, \qquad 
            (\hat{P}V_p)(\theta)  \leq \beta_p  V_p(\theta) + H_p,
    \end{align*}
    where
        $\beta_p :=  1 - \frac{mp\eta}{4}$ %
    and $H_p$ is explicitly given in the proof. 
\end{lemma}
Compared to Lemma~\ref{lem:erg}, the proof of Lemma~\ref{lem:erg_new_p2} is significantly more involved. The difficulty in this result stems from the fact that we need show contraction for higher-order moments of $\theta_k$ since we consider $V_p$ with $p\geq 1$. To overcome this difficulty, we get inspired by recent techniques that have been introduced in stochastic analysis \citep{chen2022approximation}.

\paragraph{Privacy guarantee for noisy GD.}

Equipped with these lemmas, we finally have the following DP-bound for noisy GD.

\begin{theorem}
\label{thm:dp_gd}
Let $p\in (0,\min(\frac1{2},\alpha-1))$ and $\mathcal{A}$ be the noisy GD algorithm given in \eqref{eqn:gd}, such that $\mathcal{A}(X_n) = \theta_k$ for some $k \geq 1$. Suppose that Assumptions~\ref{assump:1}, \ref{assump:grad}, \ref{asmp:interp} hold, $\eta < \min(m/(2K_1^2), 1/m, 1)$, $\sup_{x,\hat{x} \in \mathcal{X}} \|x-\hat{x}\|\leq D <\infty$. Further assume that 
$$\int_{\mathbb{R}^d} (1+ \|\theta - \vartheta_\star\|^2)^{\frac{1+p}{2}} \> p_0(\mathrm{d} \theta) \leq H_{1+p}/(1-\beta_{1+p}),$$ 
where $H_{1+p}$ and $\beta_{1+p}$ are given in Lemma~\ref{lem:erg_new_p2} and $p_0$ is the distribution of $\theta_0$. Then, for any iteration $k$ and noise scale $\sigma >0$,
$\mathcal{A}$ is $(0,\delta)$-DP with 
\begin{align}
    \label{eqn:gd_bound} \delta \leq \frac1{n} \cdot \frac{ c\left(1-\rho^k\right) C_\gamma H_{1+p}}{2(1-\rho)(1-\beta_{1+p})},
\end{align}
where $c$ and $\rho$ are given in Lemma~\ref{lem:erg}, and $C_\gamma$ is given in Lemma~\ref{lem:gamma_gd_new}.
\end{theorem}
Let us provide some remarks about this result. Firstly, Theorem~\ref{thm:dp_gd} shows that noisy GD either with heavy-tailed or Gaussian noise (the limit $\alpha \to 2$), and \emph{without projections} will achieve $(0,\delta)$-DP with $\delta = \mathcal{O}( 1/n)$. 
Secondly, let us investigate the constant $H_{1+p}$ in terms of its dependency on the dimension $d$. A quick inspection and book keeping of the constants reveal that $H_{1+p}$ has the following structure:
\begin{align*}
    H_{1+p} = \mathcal{O}\left( \left(\frac{\Gamma(\frac{d+\alpha}{2})  }{\Gamma(\frac{d}{2})}\right)^{p} + \frac{\Gamma(\frac{d+\alpha}{2})  }{\Gamma(\frac{d}{2})} \sqrt{d} + \frac{\Gamma(\frac{d+\alpha}{2})  }{\Gamma(\frac{d}{2})}  \frac{\Gamma(\frac{d+p}{2}) }{\Gamma(\frac{d}{2})} \right).
\end{align*}
By using Gautschi's inequality on the ratio of gamma functions\footnote{Let $x$ be a positive real number, and let $s \in(0,1)$. Then, Gautschi's inequality states that:
$
x^{1-s}<\frac{\Gamma(x+1)}{\Gamma(x+s)}<(x+1)^{1-s} 
$.
} \citep{gautschi1959some}, we have the following growth rates for large $d$:
\begin{align*}
\frac{\Gamma(\frac{d+\alpha}{2})  }{\Gamma(\frac{d}{2})} = \mathcal{O}\left(d^{\alpha/2}\right), \qquad \frac{\Gamma(\frac{d+p}{2}) }{\Gamma(\frac{d}{2})} = \mathcal{O}\left(d^{p/2}\right).
\end{align*}
By combining these estimates and keeping the leading order term, we observe that $H_{1+p} = \mathcal{O}\left(d^{\frac{\alpha +1}{2}}\right)$ as $d \to \infty$. This result suggests a potential benefit of heavy tails: as the noise gets heavier tailed (smaller $\alpha$), the dependency of the bound on $d$ weakens (assuming the other constants that do not have an explicit dependence on $d$ do not grow faster than $H_{1+p}$).  

In the context of generalization bounds, similar observations have been made: smaller $\alpha$ yields a lower generalization bound until $\alpha$ reaches a certain threshold; if we keep decreasing $\alpha$ beyond that threshold, the bound starts increasing \citep{raj2022algorithmic,raj2023algorithmic,pmlr-v235-dupuis24a}. Here, we only focused on the analysis of the dimension dependence of our bound for varying $\alpha$; unfortunately, as opposed to the mentioned generalization error studies, we are not able to analyze the dependence of the whole DP bound on $\alpha$, as we have two non-explicit constants, which may or may not depend on $\alpha$.

\subsection{DP guarantees without Assumption~\ref{asmp:interp}}

In this section, we will illustrate how the theory can be developed without necessitating Assumption~\ref{asmp:interp}. Essentially, Assumption~\ref{asmp:interp} provides us a point $\vartheta_\star$ which is a stable point of both $ F(\cdot, X_n)$ and $ F(\cdot, \hat{X}_n)$. In the absence of such a point, we can define the following Lyapunov functions: for $0 < p < \min(\frac1{2},\alpha-1)$ and $\alpha \in (1,2)$:
\begin{align}
    \label{eqn:v_vhat_woasmp}
    V_p(\theta) =& \left(1 + \| \theta - \theta_\star \|^2\right)^{p/2}, \\
    \label{eqn:v_vhat_woasmp2}
    \hat{V}_p(\theta) =& \left(1 + \|\theta - \hat{\theta}_\star\|^2\right)^{(1+p)/2},
\end{align}
where $\theta_\star$ and $\hat{\theta}_\star$ are stable points of $F(\cdot,X_n)$ and $F(\cdot,\hat{X}_n)$, respectively.

With this choice of $V_p$, Lemma~\ref{lem:erg} holds with the exact proof strategy. Similarly, by using the exact proof strategy of Lemma~\ref{lem:erg_new_p2}, we can show that $(\hat{P}\hat{V}_p)(\theta)  \leq \beta_p  \hat{V}_p(\theta) + H_{1+p}$ (with the exact same constants) for this new choice of $\hat{V}_p$.

On the other hand, the proof of Lemma~\ref{lem:gamma_gd_new} needs to be adapted as $V_p$ and $\hat{V}_p$ are interacting. The following result shows that the exact same conclusion of Lemma~\ref{lem:gamma_gd_new} can be obtained with a different constant, which ultimately indicates that the DP result in Theorem~\ref{thm:dp_gd} does not require Assumption~\ref{asmp:interp}.

\begin{lemma}
\label{lem:gamma_gd_withoutasmp}
      Let $\alpha \in (1,2)$, $p\in (0,\min(\frac1{2},\alpha-1))$, $P$ be the transition kernel associated with the Markov process $(\theta_k)_{k\geq 0}$ in \eqref{eqn:gd} and $\hat{P}$ be the transition kernel associated with $(\hat{\theta}_k)_{k\geq 0}$. Suppose that Assumptions~\ref{assump:1} and \ref{assump:grad} hold and further assume that $\sup_{x,\hat{x} \in \mathcal{X}} \|x-\hat{x}\|\leq D$, for some $D<\infty$. Suppose that Assumption~\ref{asmp:interp} holds and set $V_p(\theta)$ and $  \hat{V}_p(\theta)$ as in \eqref{eqn:v_vhat_woasmp} and \eqref{eqn:v_vhat_woasmp2}, respectively.
      Then, the following inequality holds:
        \begin{align*}
        \gamma =& \sup_{\theta \in \mathbb{R}^d} \frac{\|P(\theta, \cdot) - \hat{P}(\theta, \cdot)\|_{V_p}}{\hat{V}_p(\theta)} 
        \leq \frac{ \hat{C}_\gamma } {n } ,
    \end{align*}
    where $\hat{C}_\gamma>0$ is a finite constant.
\end{lemma}

\section{Privacy of Noisy Stochastic Gradient Descent}
\label{sec:sgd}

We will now analyze the DP properties of SGD, given in the recursion \eqref{eqn:sgd}. We will follow the same three-step recipe that we followed for GD. The intermediate lemmas are similar to the ones that we derived for GD, hence we report them in Appendix~\ref{sec:app:proof_sgd}. The next theorem establishes a DP bound on the noisy SGD with heavy-tailed perturbations. 

\begin{theorem}
\label{thm:dp_sgd}
Let $p\in (0,\min(\frac1{2},\alpha-1))$ and $\mathcal{A}$ be the noisy SGD algorithm given in \eqref{eqn:sgd}, such that $\mathcal{A}(X_n) = \theta_k$ for some $k \geq 1$. Suppose that Assumptions~\ref{assump:1}, \ref{assump:grad}, \ref{asmp:interp} hold, $\eta < \min(m/(2K_1^2), 1/m, 1)$, $\sup_{x,\hat{x} \in \mathcal{X}} \|x-\hat{x}\|\leq D <\infty$. Further assume that 
$$\int_{\mathbb{R}^d} (1+ \|\theta - \vartheta_\star\|^2)^{\frac{1+p}{2}} \> p_0(\mathrm{d} \theta) \leq H_{1+p}/(1-\beta_{1+p}),$$ 
where $H_{1+p}$ and $\beta_{1+p}$ are given in Lemma~\ref{lem:erg_new_p2} and $p_0$ is the distribution of $\theta_0$. Then, for any iteration $k$, batch-size $b$, and noise scale $\sigma >0$,
$\mathcal{A}$ is $(0,\delta)$-DP with 
\begin{align}
    \label{eqn:gd_bound_2} \delta \leq \frac1{n} \cdot \frac{ c\left(1-\rho^k\right) C_\gamma H_{1+p}}{2(1-\rho)(1-\beta_{1+p})},
\end{align} 
where $c$, $\rho$, $C_\gamma$ are defined similarly to the ones in Theorem~\ref{thm:dp_gd}.
\end{theorem}
We omit the proof of Theorem~\ref{thm:dp_sgd} as it follows the same lines as the proof of Theorem~\ref{thm:dp_gd}, except that we need to invoke the lemmas proven in Appendix~\ref{sec:app:proof_sgd} instead of the ones in the previous section. 

We note that we obtain the exact same DP guarantee for noisy GD and noisy SGD: it achieves $(0,\delta)$-DP with $\delta = \mathcal{O}(1/n)$. Here, an important observation is that the batch-size does not appear directly in the bound given in Theorem~\ref{thm:dp_sgd}. This is due to the fact that the global stable point condition (Assumption~\ref{asmp:interp}) allows us to simplify some terms so that the randomness introduced by the stochastic gradients can be handled uniformly. Hence the effect of the stochastic gradients is not reflected in the bound. If we did not use Assumption~\ref{asmp:interp}, due to the additional noise coming from minibatches, the analysis of noisy SGD would introduce an technical challenge, which could be handled at the cost of larger constants.

\subsection{Comparison to prior work when $\alpha =2$.}

As our results are the first DP guarantees for (S)GD with heavy-tailed noise to our knowledge, we are not able to perform a comparison for the heavy-tailed case. Hence, we will attempt to compare our bounds to the prior work when the noise is Gaussian, which corresponds to $\alpha =2$ in our framework. In the Gaussian noise case, under different assumptions on the loss function \citep{chourasia2021differential,altschuler2022privacy,ryffel2022differential} proved DP guarantees by using the notion ($\mathrm{a},\varepsilon$)-R\'{e}nyi DP \citep{mironov2017renyi}.
They showed that noisy (S)GD achieves $(\mathrm{a},\varepsilon$)-R\'{e}nyi DP with $\varepsilon = \mathcal{O}(\mathrm{a}/n^2)$.

To be able to have a fair comparison, we need to convert our results to ($\mathrm{a},\varepsilon$)-R\'{e}nyi DP. Setting $\delta = C/n$ for $C >0$, by \citet[Theorem 4]{asoodeh2021three}, our $(0, \delta)$-DP bounds imply ($\mathrm{a},\varepsilon$)-R\'{e}nyi DP with   $\mathrm{a} = n/C$ and $\varepsilon = \mathcal{O}(1/n)$. 
Hence, when we set $\mathrm{a} = n/ C$ in the bounds of prior work, we observe that they also obtain $\varepsilon = \mathcal{O}(1/n)$. This shows that, even though our approach does not necessitate projections and can cover heavy tails as well, we recover the order of the existing bounds.

\section{Conclusion}

We established DP guarantees for noisy gradient descent and stochastic gradient descent under $\alpha$-stable perturbations, which encompass both heavy-tailed and Gaussian distributions. By using recent tools from Markov chain stability theory, we showed that the algorithms achieve a time-uniform (i.e., does not depend on the number of iterations) $(0, \mathcal{O}(1/n))$-DP for a broad class of loss functions, which can be non-convex. Contrary to prior work, we showed that clipping the iterates is not required for DP once the loss function and the data satisfy mild assumptions.  
Our results suggest that the heavy-tailed mechanism is a viable alternative to its light-tailed counterparts, especially given the other benefits of heavy tails. Finally, we hope that our proof technique would inspire future developments in the field.

\section*{Acknowledgments}
Umut \c{S}im\c{s}ekli's research is supported European Research Council Starting Grant
DYNASTY – 101039676.
Mert G\"urb\"uzbalaban's research is partially supported by the Office of Naval Research under Award Numbers N00014-21-1-2244 and N00014-24-1-2628; and by the grants National Science Foundation (NSF)  CCF-1814888 and NSF DMS-2053485. Lingjiong Zhu is partially supported by the grants NSF DMS-2053454, NSF DMS-2208303.

\appendix

\begin{center}
\bf Appendix

\end{center}

The Appendix is organized as follows:
\begin{itemize}
    \item In Appendix~\ref{sec:exp_constants}, we compute the constants required for our assumptions for regularized linear and logistic regression problems. 
    \item In Appendix~\ref{sec:app:proof_gd}, we provide the proofs of the results of privacy of noisy GD in Section~\ref{sec:gd} in the main paper.
    \item In Appendix~\ref{sec:app:proof_sgd}, we provide the proofs of the results of privacy of noisy SGD in Section~\ref{sec:sgd} in the main paper.
    \item We present some additional technical lemmas in Appendix~\ref{sec:technical}.
\end{itemize}

\section{Computation of the Constants for Assumptions~\ref{assump:1} and \ref{assump:grad}}
\label{sec:exp_constants}

\subsection{Linear Regression}

In this section, we will derive the constants required for Assumptions~\ref{assump:1} and \ref{assump:grad} for a regularized linear regression problem.

\begin{proposition}
    \label{lemma-assump-satisfied-ridge-regression} Consider ridge regression with quadratic loss $f(\theta,x) := \frac{1}{2}(\theta^{\top} a - b)^2 +\frac{\lambda}{2} \|\theta\|^2$ where $x = (a,b)$ is the input-output data pair with $a \in \mathbb{R}^{d}$ and $b\in \mathbb{R}$ and $\lambda>0$ is a regularization parameter. Given $p \in (0,1)$, let $R_p>0$ be a constant such that the data $\|x\|\leq R_p$ with probability $1-p$. Then, Assumption~\ref{assump:1} holds with probability $1-2p$ with constants $K_1 = R_p^2 + \lambda $ and $K_2 = 2R_p$. Furthermore, Assumption~\ref{assump:grad} holds with constants $B = R_p^2$ with probability $1-p$, for $m = \lambda$ and for any $K\geq 0$.
\end{proposition}

\begin{proof}[Proof of Proposition~\ref{lemma-assump-satisfied-ridge-regression}] First, we note that $\nabla f(\theta,x) = aa^{\top}\theta - ba +\lambda \theta $. If we consider the data point $\hat x = (\hat a, \hat b)$, then 
 \begin{align*} 
 \| \nabla f(\theta,x) - \nabla f(\theta, \hat x) \| 
 &= \left\| \left(aa^{\top} - \hat a \hat a^{\top}\right)\theta - \left(b a -  \hat b\hat a\right)  \right\| \\
 &\leq \left\| \left(aa^{\top} - a\hat a^{\top} + a \hat a^{\top} - \hat a \hat a^{\top}\right)\theta\right\| + \left\|b a -  \hat b\hat a  \right\| \\
 &\leq (\|a\| \|a-\hat a\| + \|\hat a\|\|a-\hat a\|) \|\theta\| + \|b\| \left\|a - \hat a\right\| + \|\hat a\| \|b-\hat b\| \\
 &\leq 2R_p \|x-\hat x\| \|\theta\| + 2R_p \|x-\hat x\|, 
 \end{align*}
provided that $\|x\| \leq R_p$ and $\|\hat x \|\leq R_p$. This is the case with probability (at least) $1-2p$. Similarly, 
 \begin{eqnarray*} 
 \left\| \nabla f(\theta, \hat x) - \nabla f(\hat\theta, \hat x)\right\| = \left\|  \hat a \hat a^{\top} (\theta - \hat\theta) + \lambda (\theta-\hat \theta) \right\| \leq (R_p^2 + \lambda) \|\theta - \hat \theta\|,
 \end{eqnarray*}
with probability $1-p$ when $\|\hat x \|\leq R_p$. Therefore, we conclude that
\begin{eqnarray*}
\left\| \nabla f(\theta,x) - \nabla f(\hat\theta, \hat x)\right\| 
&\leq& \left\| \nabla f(\theta,x) - \nabla f(\theta, \hat x) \right\| + \left\| \nabla f(\theta, \hat x) - \nabla f(\hat\theta, \hat x)\right\| 
\\
&\leq& (R_p^2 +\lambda) \|\theta - \hat \theta\| + 2R_p \|x-\hat x\|(1+\|\theta\|),
\end{eqnarray*}
with probability $1-2p$. This proves that Assumption~\ref{assump:1} holds with  with probability $1-2p$ with constants $K_1 = R_p^2 +\lambda$ and $K_2 = 2R_p$. In regards to Assumption~\ref{assump:grad}, note that  
\begin{align*}
&\Vert\nabla f(0,x)\Vert = \Vert ba \| \leq R_p^2, \quad\mbox{ with probability } 1-p,
\\
&\left\langle\nabla f(\theta_{1},x)-\nabla f(\theta_{2},x),\theta_{1}-\theta_{2}\right\rangle = \left\langle (aa^{\top}+\lambda I) (\theta_1 - \theta_2) ,\theta_{1}-\theta_{2}\right\rangle 
\geq \lambda \Vert\theta_{1}-\theta_{2}\Vert^{2}, 
\end{align*} 
for any $\theta_{1},\theta_{2}\in\mathbb{R}^{d}$, where $I$ is the identity matrix. The proof is complete.
\end{proof}

\subsection{Logistic Regression}

In this section, we will derive the constants required for Assumptions~\ref{assump:1} and~\ref{assump:grad} for a regularized logistic regression problem. 
To fit the logistic regression problem into our framework, we will need to come up with an equivalent definition for the loss function. Let us start with the conventional definition of the logistic regression problem:   Let $x = (u, z)$, where $u \in \mathbb{R}^{d}$ is the feature vector and $z \in \{-1, +1\}$ is the binary response. The loss function is defined as
\[
f(\theta, x) = \log \left(1 + e^{-z u^{\top} \theta}\right) + \frac{\lambda}{2} \| \theta \|^{2}, \quad z \in \{-1, 1\},\,\, u, \theta \in \mathbb{R}^{d},
\]
where $\lambda>0$ is a regularization parameter.

The product $z u$ is arguably artificial. We can reduce the data points $(u, z)$ of logistic regression to the product of the feature $u$ and the label $z$, i.e., $uz$, since the loss function of the model can be equivalently written as $\log \left(1 + e^{-(zu)^{\top} \theta}\right)$. 

Therefore, we will instead let $x = z u$ and define the logistic model in terms of the product $x$ and $\theta$ only, which is formalized in the following proposition.

\begin{proposition}
\label{prop:log_reg}

Consider the logistic regression problem with $\ell_2$ regularization: $f(\theta, x) := \log (1+ \exp(- x^\top \theta)) + \frac{\lambda}{2} \|\theta\|^2$, where $x = u z$ is the product of the feature $u \in \mathbb{R}^d$ and the label $z \in \{-1,1\}$, and $\lambda >0$ is the regularization parameter. Assume that $\|x\| \leq R$ for every $x \in \mathcal{X}$. Then, Assumption~\ref{assump:1} holds with constants $K_1 = R^2+ \lambda $ and $K_2 = \max\{1, R\}$. Furthermore, Assumption~\ref{assump:grad} holds with constants $B = R/2$, $m = \lambda$ and for any $K\geq 0$.
\end{proposition}

\begin{proof}%
    For every $x,x'\in\mathcal{X}$ and $\theta,\theta'\in\mathbb{R}^{d}$, we would like to provide an upper bound for $\| \nabla f(\theta', x')- \nabla f(\theta, x) \|$. Using the triangular inequality, we have that
\begin{align}
\| \nabla f(\theta', x')- \nabla f(\theta, x) \|  \leq \| \nabla f(\theta, x')- \nabla f(\theta, x) \| + \| \nabla f(\theta, x') -\nabla f(\theta', x') \|.  \label{eq: triangular}
\end{align}
For the first term on the right-hand side of \eqref{eq: triangular}, we have 
\begin{align*}
\| \nabla f(\theta, x')- \nabla f(\theta, x) \| &= \left\| \frac{x e^{-x^{\top} \theta}}{1 + e^{-x^{\top} \theta}} - \frac{x' e^{-x'^{\top} \theta}}{1 + e^{-x'^{\top} \theta}} \right\| \\
& \leq \left\| x \left( \frac{e^{-x^{\top} \theta}}{1 + e^{-x^{\top} \theta}} - \frac{e^{-x'^{\top} \theta}}{1 + e^{-x'^{\top} \theta}} \right) \right\| + \left\| (x - x') \frac{e^{-x'^{\top} \theta}}{1 + e^{-x'^{\top} \theta}} \right\| \\
& \leq \| x \| \left\vert \frac{e^{-x^{\top} \theta}}{1 + e^{-x^{\top} \theta}} - \frac{e^{-x'^{\top} \theta}}{1 + e^{-x'^{\top} \theta}} \right\vert + \| x - x' \| \\
& =  \| x \| \left\vert \frac{1}{1 + e^{x^{\top} \theta}} - \frac{1}{1 + e^{x'^{\top} \theta}} \right\vert + \| x - x' \| \\
& \leq  \| x \| \left\vert \log \left(1 + e^{x^{\top} \theta}\right) - \log \left(1 + e^{x'^{\top} \theta}\right) \right\vert + \| x - x' \|,
\end{align*}
where the last line is since for $0< a, b < 1$ we have $|a - b| \leq |\log a - \log b| = |\log (1/a) - \log (1/b)| $. 
Using, e.g., \citet[Section 4.2]{Yildirim_and_Ermis_2019}, we have
\[
\left\vert \log \left(1 + e^{x^{\top} \theta}\right) - \log \left(1 + e^{x'^{\top} \theta}\right) 
\right\vert \leq \left|\theta^{\top} \left(x - x'\right)\right| \leq \|\theta\| \| x - x' \|.
\]
Therefore, for the first term on the right-hand side in \eqref{eq: triangular} we arrive at
\begin{align}
\left\| \nabla f(\theta, x')- \nabla f(\theta, x) \right\| &\leq \| x \| \|\theta\| \| x - x' \| + \| x - x' \| \nonumber \\
&\leq \max \{1, \| x \| \} \| x - x' \| (\| \theta\| + 1). \label{eq: triangular-term 1 final}
\end{align}
For the second term on the right-hand side in \eqref{eq: triangular}, we have
\begin{align}
\left\| \nabla f(\theta, x')  - \nabla f(\theta', x') \right\| &=  \left\| \frac{x' e^{-x'^{\top} \theta'}}{1 + e^{-x'^{\top} \theta'}} - \frac{x' e^{-x'^{\top} \theta}}{1 + e^{-x'^{\top} \theta}} + \lambda \left(\theta - \theta'\right) \right\| \nonumber \\
& \leq \left\| x' \left( \frac{e^{-x'^{\top} \theta'}}{1 + e^{-x'^{\top} \theta'}} - \frac{e^{-x'^{\top} \theta}}{1 + e^{-x'^{\top} \theta}} \right) \right\| + \lambda \| \theta' - \theta \| \nonumber\\
& = \| x' \| \left\vert \frac{1}{1 + e^{x'^{\top} \theta'}} - \frac{1}{1 + e^{x'^{\top} \theta}} \right\vert + \lambda \| \theta' - \theta \| \nonumber\\
& \leq  \| x' \| \left\vert \log \left(1 + e^{x'^{\top} \theta'}\right) - \log \left(1 + e^{x'^{\top} \theta}\right) \right\vert + \lambda \| \theta' - \theta \|  \nonumber\\
& \leq \| x' \| \| x' \| \| \theta' - \theta \| + \lambda \| \theta' - \theta \| \nonumber\\
&= (\| x' \|^{2} + \lambda) \| \theta' - \theta \|, \label{eq: triangular-term 2 final}
\end{align}
where we have followed similar lines to those for the first term. Combining \eqref{eq: triangular-term 1 final} and \eqref{eq: triangular-term 2 final}, we end up with
\begin{align*}
\left\| \nabla f(\theta', x')- \nabla f(\theta, x) \right\| &\leq \max \{1, \| x \| \} \| x - x' \| (\| \theta\| + 1) +  (\| x' \|^{2} + \lambda)\| \theta' - \theta \| \\
& \leq \max \{1, \| x \| \} \| x - x' \| (\| \theta\| + \|\theta'\| + 1) +  (\| x' \|^{2} + \lambda) \| \theta' - \theta \|. 
\end{align*}
Since $x$ is bounded in norm, letting $K_{1} = R^2 + \lambda$ and $K_{2} = \max \{1, R \}$, we have \eqref{eq: regularity assumption - 1} in Assumption~\ref{assump:1}.

On the other hand, since the loss function $f(\cdot, x)$ is $\lambda$-strongly convex for every $x$, the second inequality in Assumption~\ref{assump:grad} is satisfied with $m = \lambda$ and $K \geq 0$. Finally, we have that 
\begin{align*}
    \|\nabla f(0,x)\| = \|x\|/2 \leq R/2.
\end{align*}
Hence, the first inequality in  Assumption~\ref{assump:grad} holds with $B = R/2$.
This completes the proof.
\end{proof}

\section{Proofs of the Results of Section~\ref{sec:gd}}
\label{sec:app:proof_gd}

\subsection{Proof of Lemma~\ref{lem:gamma_gd_new}}

Let $\mathcal{S}(\alpha, \beta, \gamma, \delta)$ denote the univariate $\alpha$-stable distribution with the following characteristic function \citep[Definition 1.3]{nolan2020univariate}: if $ Z \sim \mathcal{S}(\alpha, \beta, \gamma, \delta)$
    \begin{align*}
        \mathbb{E} \exp (\iu u Z)= \begin{cases}\exp \left(-\gamma^\alpha|u|^\alpha\left[1+\iu \beta\left(\tan \frac{\pi \alpha}{2}\right)(\operatorname{sign} u)\left(|\gamma u|^{1-\alpha}-1\right)\right]+\iu \delta u\right) & \alpha \neq 1 \\ \exp \left(-\gamma|u|\left[1+\iu \beta \frac{2}{\pi}(\operatorname{sign} u) \log (\gamma|u|)\right]+\iu \delta u\right) & \alpha=1\end{cases}.
    \end{align*}

\begin{lemma}
\label{lem:v_norm_cond}
    Let $V:\mathbb{R}^d\to [1,\infty)$ be a measurable function,  $\alpha \in (1,2)$, $P$ be the transition kernel associated with the Markov process $(\theta_k)_{k\geq 0}$ in \eqref{eqn:gd}, and $\hat{P}$ be the transition kernel associated with $(\hat{\theta}_k)_{k\geq 0}$. Then, the following inequality holds:
    \begin{align}
  \nonumber \|P(\theta, \cdot) - \hat{P}(\theta, \cdot)\|_V \leq \int_{\mathbb{R}_+} \sqrt{ 2\left\{\mu_{\theta,\phi} \left(V^2\right)+\hat{\mu}_{\theta,\phi}\left(V^2\right)\right\} \mathrm{KL}(\mu_{\theta,\phi} \mid \hat{\mu}_{\theta,\phi})}  \>  p(\phi)  \mathrm{d}\phi,
\end{align}
where $p(\phi)$ is the probability density function of $\mathcal{S}\left(\frac{\alpha}{2}, 1, \left( \cos \frac{\pi \alpha}{4} \right)^{2/\alpha},0\right)$, and
\begin{align*}
 \mu_{\theta,\phi} =& \> \mathcal{N}\Bigl(\theta -\eta \nabla F(\theta, X_n), \phi \sigma^2 \mathrm{I}_d \Bigr), \\   
 \hat{\mu}_{\theta,\phi} =& \> \mathcal{N}\left(\theta -\eta \nabla F(\theta, \hat{X}_n), \phi \sigma^2 \mathrm{I}_d \right),
\end{align*}
 are Gaussian measures
with $\mathrm{I}_d$ being the $d\times d$ identity matrix.  
\end{lemma}

\begin{proof}[Proof of Lemma~\ref{lem:v_norm_cond}]
 For $\theta \in \mathbb{R}^d$, let $p_\theta$ and $\hat{p}_\theta$ denote the densities associated with $P(\theta, \cdot)$ and $\hat{P}(\theta, \cdot)$, respectively. Since the injected noise $\xi_1$ is rotationally invariant stable distributed, by \citet[Proposition 2.5.2]{ST1994}, we have the following scale-mixture of Gaussian representation for $p_\theta$ and $\hat{p}_\theta$:
    \begin{align*}
        p_\theta(x) &= \int_{\mathbb{R}_+} p_\theta(x|\phi) p(\phi) \mathrm{d}\phi, \\
        \hat{p}_\theta(x) &= \int_{\mathbb{R}_+} \hat{p}_\theta(x|\phi) p(\phi) \mathrm{d}\phi,
    \end{align*}
    where for $\phi  \in \mathbb{R}_+$, $p_\theta(x|\phi)$ and $\hat{p}_\theta(x|\phi)$ are the probability density functions corresponding to $\mu_{\theta,\phi}$ and $\hat{\mu}_{\theta,\phi}$ respectively. 
By using this decomposition and the non-variational definition of the $V$-norm (i.e., $\|\mu\|_V = |\mu|(V)$), we have that:
\begin{align}
  \nonumber \|P(\theta, \cdot) - \hat{P}(\theta, \cdot)\|_V =& \int_{\mathbb{R}^d} V(y) |P(\theta, \mathrm{d}y) - \hat{P}(\theta, \mathrm{d}y)| \\
  \nonumber  =& \int_{\mathbb{R}^d} V(y) |p_\theta(y) - \hat{p}_\theta(y)| \mathrm{d}y \\
  \nonumber  =& \int_{\mathbb{R}^d} V(y) \left| \int_{\mathbb{R}_+} \Bigl( p_\theta(y|\phi) - \hat{p}_\theta(y|\phi) \Bigr) p(\phi) \mathrm{d}\phi \right| \mathrm{d}y \\
  \label{eqn:proof:gamma_gd_new_1} \leq & \int_{\mathbb{R}_+} \left[ \int_{\mathbb{R}^d} V(y) \left|   p_\theta(y|\phi) - \hat{p}_\theta(y|\phi)   \right| \mathrm{d}y  \right] p(\phi) \mathrm{d}\phi \\
  \label{eqn:proof_onestep3}
    =& \int_{\mathbb{R}_+} \|\mu_{\theta,\phi}  - \hat{\mu}_{\theta,\phi} \|_V \>  p(\phi)  \mathrm{d}\phi,
\end{align}
where \eqref{eqn:proof:gamma_gd_new_1} follows from Tonelli's theorem, and $\mu_{\theta,\phi}$ and $\hat{\mu}_{\theta,\phi}$ are the Gaussian measures associated with the densities $p_\theta(\cdot|\phi)$ and $\hat{p}_\theta(\cdot|\phi)$, respectively. 

By Lemma~\ref{lem:vnorm_kl}, we have that
\begin{align}
    \label{eqn:proof_onestep_2}
    \|\mu_{\theta,\phi}- \hat{\mu}_{\theta,\phi}\|_V \leq \sqrt{2}\left\{\mu_{\theta,\phi} \left(V^2\right)+\hat{\mu}_{\theta,\phi}\left(V^2\right)\right\}^{1 / 2} \mathrm{KL}^{1 / 2}(\mu_{\theta,\phi} \mid \hat{\mu}_{\theta,\phi}) .
\end{align}
Using this inequality in \eqref{eqn:proof_onestep3} yields the desired result. This completes the proof. 
\end{proof}

\begin{proof}
[Proof of Lemma~\ref{lem:gamma_gd_new}]
We start by estimating the $V_p$-norm of the difference between one-step transition kernels $P$ and $\hat{P}$. 

By Lemma~\ref{lem:v_norm_cond}, we have that 
\begin{align}
     \|P(\theta, \cdot) - \hat{P}(\theta, \cdot)\|_{V_p} \leq \int_{\mathbb{R}_+} \sqrt{ 2\left\{\mu_{\theta,\phi} \left(V_p^2\right)+\hat{\mu}_{\theta,\phi}\left(V_p^2\right)\right\} \mathrm{KL}\left(\mu_{\theta,\phi} \mid \hat{\mu}_{\theta,\phi}\right)}  \>  p(\phi)  \mathrm{d}\phi .
     \label{eqn:proof_onestep_tmp}
\end{align}
Since the measures $\mu_{\theta,\phi}$ and $\hat{\mu}_{\theta,\phi}$ are Gaussian, the KL term in \eqref{eqn:proof_onestep_tmp} admits the following analytical expression (see e.g., \citet[Fact A.3]{arbas2023polynomial}): 
\begin{align}
    \mathrm{KL}\left(\mu_{\theta,\phi} \mid \hat{\mu}_{\theta,\phi}\right) = \frac{\eta^2}{\phi \sigma^2} \|\nabla F(\theta, X_n) - \nabla F(\theta, \hat{X}_n)\|^2.
\end{align}
Recall that $X_{n}=\{x_{1},\ldots,x_{n}\}$
and $\hat{X}_{n}=\{\hat{x}_{1},\ldots,\hat{x}_{n}\}$
differ by at most one data point, and without loss of generality
we assume that $x_{j}=\hat{x}_{j}$ for every $j\neq i$
for some $i\in\{1,\ldots,n\}$.
By using the definition of $F$, invoking Assumption~\ref{assump:1}, and using the fact that $\mathcal{X}$ is bounded, we further have that:
    \begin{align}
          \left(\mathrm{KL}\left(\mu_{\theta,\phi} \mid \hat{\mu}_{\theta,\phi}\right)\right)^{1/2} =& \frac{ \eta} {n \sigma \sqrt{\phi}} \left\|\nabla f(\theta, x_i) - \nabla f(\theta, \hat{x}_i)\right\| \nonumber\\
         \leq&  \frac{ 2 K_2 \eta } {n \sigma }   \|x_i - \hat{x}_i\| (\|\theta\|+1) \nonumber \\
         \leq& \frac1{\sqrt{\phi}} \underbrace{\frac{ 2 K_2 D \eta} {n \sigma } }_{=: C_1} (\|\theta\|+1) \label{defn:C:1}\\
         \leq& \phi^{-1/2} C_1 (1 + \|\theta - \vartheta_\star\|) + \phi^{-1/2} C_1 \|\vartheta_\star\|.\nonumber
    \end{align}
We now proceed to the estimation of the expectations of the  Lyapunov function under $\mu_{\theta,\phi}$ and $\hat{\mu}_{\theta,\phi}$. Recalling that $2p <1$ such that $(x+y)^{2p}\leq x^{2p}+y^{2p}$ for any $x,y\geq 0$, we have that:
\begin{align}
    \nonumber \mu_{\theta,\phi}(V_p^2) =& \int_{\mathbb{R}^d} (1 + \|y - \vartheta_\star\|^2)^p \> p_\theta(y | \phi) \>  \mathrm{d}y \\
    \nonumber \leq& 1 + \int_{\mathbb{R}^d} \|y - \vartheta_\star\|^{2p} \> p_\theta(y | \phi) \>  \mathrm{d}y \\
    \nonumber =& 1 + \int_{\mathbb{R}^d} \|y - \vartheta_\star - (\theta - \eta \nabla F(\theta,X_n)) + (\theta - \eta \nabla F(\theta,X_n)) \|^{2p} \> p_\theta(y | \phi) \>  \mathrm{d}y \\
    \leq& 1 +  \|\theta - \vartheta_\star - \eta \nabla F(\theta,X_n)) \|^{2p} + \int_{\mathbb{R}^d} \|y  - (\theta - \eta \nabla F(\theta,X_n)) \|^{2p} \> p_\theta(y | \phi) \>  \mathrm{d}y .\label{three:terms}
\end{align}
Let us first focus on the second term in \eqref{three:terms}. We can compute that:
    \begin{align}
        \nonumber \|\theta  - \vartheta_\star - \eta \nabla F(\theta,X_n) \|^2 =& \|\theta  - \vartheta_\star\|^2 - 2\eta \langle \theta  - \vartheta_\star,  \nabla F(\theta,X_n) - \nabla F(\vartheta_\star,X_n) \rangle \\
        \nonumber &\qquad \qquad \qquad \qquad \qquad \qquad + \eta^2 \|\nabla F(\theta,X_n) - \nabla F(\vartheta_\star,X_n)\|^2 \\
        \label{eqn:proof_onestepnew1}
        \leq & \left(1-2\eta m + \eta^2 K_1^2\right) \|\theta  - \vartheta_\star\|^2 + 2\eta K,
    \end{align}
    where in \eqref{eqn:proof_onestepnew1} we used Assumptions~\ref{assump:1} and \ref{assump:grad}.

For the third term in \eqref{three:terms}, we notice that:
\begin{align*}
    \int_{\mathbb{R}^d} \|y  - (\theta - \eta \nabla F(\theta,X_n)) \|^{2p} \> p_\theta(y | \phi) \>  \mathrm{d}y =& \phi^p \sigma^{2p} \mathbb{E}\|G\|^{2p}, \\
    =:& \phi^p \sigma^{2p} C(p), 
\end{align*}
where $G$ is a standard Gaussian random vector, and
\begin{equation}\label{defn:C:p}
C(p)=\mathbb{E}\|G\|^{2p}=2^{p}\frac{\Gamma(p+\frac{d}{2})}{\Gamma(\frac{d}{2})}, 
\end{equation}
where we used the fact that $\Vert G\Vert^{2}$ is a chi-square distribution of degree $d$ and the expression of the $p$-th moment of a chi-square distribution.

By combining these computations, and using the inequality that with $2p<1$, $(x+y)^{2p}\leq x^{2p}+y^{2p}$ for any $x,y\geq 0$, we obtain that:
\begin{align}
    \nonumber \mu_{\theta,\phi}(V_p^2) \leq& 1 +  \left(1-2\eta m + \eta^2 K_1^2\right)^p \|\theta  - \vartheta_\star\|^{2p} + (2\eta K)^p +  \phi^p \sigma^{2p} C(p) .
\end{align}
By following similar steps, we also obtain:
\begin{align}
    \nonumber \hat{\mu}_{\theta,\phi}(V_p^2) \leq& 1 +  \left(1-2\eta m + \eta^2 K_1^2\right)^p \|\theta  - \vartheta_\star\|^{2p} + (2\eta K)^p +  \phi^p \sigma^{2p} C(p) .
\end{align}
Using these estimates for the integrand in \eqref{eqn:proof_onestep_tmp}, we get:
\begin{align*}
     &\sqrt{ 2\left\{\mu_{\theta,\phi} \left(V_p^2\right)+\hat{\mu}_{\theta,\phi}\left(V_p^2\right)\right\} \mathrm{KL}(\mu_{\theta,\phi} \mid \hat{\mu}_{\theta,\phi})} 
     \\
     &\leq 2 \left\{ 1 +  \left(1-2\eta m + \eta^2 K_1^2\right)^p \|\theta  - \vartheta_\star\|^{2p} + (2\eta K)^p +  \phi^p \sigma^{2p} C(p)  \right\}^{1 / 2} \\
     & \qquad\qquad \cdot\left(\phi^{-1/2} C_1 (1 + \|\theta - \vartheta_\star\|) +\phi^{-1/2} C_1 \|\vartheta_\star\| \right)\\
     &\leq 2\phi^{-1/2}  C_1 \left\{ 1 +  \left(1-2\eta m + \eta^2 K_1^2\right)^{p/2} \|\theta  - \vartheta_\star\|^{p} + (2\eta K)^{p/2} +  \phi^{p/2} \sigma^{p} \sqrt{C(p)}  \right\} \\
     & \qquad\qquad \cdot\left\{  (1 + \|\theta - \vartheta_\star\|) + \|\vartheta_\star\| \right\}.
\end{align*}
Next, we denote 
\begin{equation}\label{defn:C:2:C:3}
C_2 := \mathbb{E}[\phi^{-1/2}]\qquad\text{and}
\qquad
C_3 := \mathbb{E}[\phi^{(p-1)/2}].
\end{equation}
By Equation~(12) of \citet{matsui2016fractional}, we have that 
    \begin{align}
         \mathbb{E}[\phi^{-1/2}] =& \left( \Gamma(3/2) \cos (-\pi/4) \right)^{-1} \Gamma\left(1 + \frac{1}{\alpha}\right) \left(1 + \tan^2 \frac{\pi \alpha}{4}\right)^{-1/(2\alpha)} \cos \left( -\pi/4 \right) \gamma_\phi^{-1/2}\nonumber \\
         =& \frac{2}{\sqrt{\pi}} \Gamma\left(1 + \frac{1}{\alpha}\right) \left( \cos\frac{\pi \alpha}{4}\right)^{1/\alpha} \gamma_\phi^{-1/2} \nonumber\\
         =& \frac{2}{\sqrt{\pi}} \Gamma\left(1 + \frac{1}{\alpha}\right),\label{defn:C:2}
\end{align}
where $\gamma_{\phi} := \left( \cos \frac{\pi \alpha}{4} \right)^{2/\alpha}$ and we used the identities $\Gamma(3/2)=\sqrt{\pi}/2$ and $1 + \tan^2(x) = 1/\cos^2(x)$. 
Similarly, we can compute that
\begin{align}
\mathbb{E}[\phi^{(p-1)/2}]&=\left(\Gamma\left(\frac{3-p}{2}\right)\cos\left(\frac{(p-1)\pi}{4}\right)\right)^{-1}\nonumber
\\
&\qquad\qquad\qquad\cdot\Gamma\left(\frac{1+\alpha-p}{\alpha}\right)\left(1+\tan^{2}\frac{\pi\alpha}{4}\right)^{\frac{p-1}{2\alpha}}\cos\left(\frac{(p-1)\pi}{4}\right)\gamma_{\phi}^{\frac{p-1}{2}}\nonumber
\\
&=\frac{\Gamma\left(\frac{1+\alpha-p}{\alpha}\right)}{\Gamma\left(\frac{3-p}{2}\right)}\left(\cos\frac{\pi\alpha}{4}\right)^{\frac{1-p}{\alpha}}\left(\left( \cos \frac{\pi \alpha}{4} \right)^{2/\alpha}\right)^{\frac{p-1}{2}}\nonumber
\\
&=\frac{\Gamma\left(\frac{1+\alpha-p}{\alpha}\right)}{\Gamma\left(\frac{3-p}{2}\right)}.\label{defn:C:3}
\end{align}
In addition, by denoting 
\begin{equation}\label{defn:C:4:C:5}
C_4 := 2C_1 C_2, 
\qquad
C_5:= 2C_1 C_3, 
\end{equation}
and using the above inequality again in \eqref{eqn:proof_onestep_tmp}, we obtain:
\begin{align}
  \nonumber &\|P(\theta, \cdot) - \hat{P}(\theta, \cdot)\|_{V_p}
  \\
  &\leq  \left(C_4  +  C_4 \left(1-2\eta m + \eta^2 K_1^2\right)^{p/2} \|\theta  - \vartheta_\star\|^{p} + C_4 (2\eta K)^{p/2} + C_5  \sigma^{p} \sqrt{C(p)}\right)\nonumber   \\
     & \qquad\qquad \cdot\left\{  (1 + \|\theta - \vartheta_\star\|) + \|\vartheta_\star\| \right\}\nonumber \\
   &= \Bigl( C_6 \| \theta -\vartheta_\star \|^p + C_7 \Bigr) \Bigl(  (1 + \|\theta - \vartheta_\star\|) + \|\vartheta_\star\| \Bigr),
\end{align}
where
\begin{align}
    C_6 :=& C_4 \left(1-2\eta m + \eta^2 K_1^2\right)^{p/2},\label{defn:C:6} \\
    C_7:= & C_4 \sqrt{2}  + C_4 (2\eta K)^{p/2} + C_5  \sigma^{p} \sqrt{C(p)}. \label{defn:C:7}
\end{align}
We first consider the case $\|\theta- \vartheta_\star\| \leq 1$.
\begin{align}
    &\sup_{\theta \in \mathbb{R}^d, \|\theta- \vartheta_\star\| \leq 1 } \frac{ \|P(\theta, \cdot) - \hat{P}(\theta, \cdot)\|_{V_p} }{\hat{V}_p(\theta)} \nonumber \\
    &\leq \sup_{\theta \in \mathbb{R}^d, \|\theta- \vartheta_\star\| \leq 1 } \frac{ \Bigl( C_6 \| \theta -\vartheta_\star \|^p + C_7 \Bigr) \Bigl(  (1 + \|\theta - \vartheta_\star\|) + \|\vartheta_\star\| \Bigr) }{(1 + \|\theta- \vartheta_\star\|^2)^{(1+p)/2}} \nonumber  \\
    &\leq \sup_{\theta \in \mathbb{R}^d, \|\theta- \vartheta_\star\| \leq 1 } \frac{ \Bigl( C_6  + C_7 \Bigr) \Bigl(  2 + \|\vartheta_\star\| \Bigr) }{(1 + \|\theta- \vartheta_\star\|^2)^{(1+p)/2} }\nonumber\\
    &\leq (C_6+C_7)(2+\|\vartheta_\star\|).
\end{align}
We now proceed with the case $\|\theta-\vartheta_\star \|>1$. First notice that, in this case, we have that:
\begin{align*}
    &\Bigl( C_6 \| \theta -\vartheta_\star \|^p + C_7 \Bigr) \Bigl(  (1 + \|\theta - \vartheta_\star\|) + \|\vartheta_\star\| \Bigr) 
    \\
    &= C_6(1+ \|\vartheta_\star\|) \| \theta -\vartheta_\star \|^p + C_6 \| \theta -\vartheta_\star \|^{1+p}+C_7 \| \theta -\vartheta_\star \| + C_7 (1+ \|\vartheta_\star\|) \\
    &\leq \Bigl(C_7 + C_6(2+ \|\vartheta_\star\|) \Bigr) \| \theta -\vartheta_\star \|^{1+p} 
     + C_7 (1+ \|\vartheta_\star\|)\\
     &\leq \Bigl(C_7 + C_6(2+ \|\vartheta_\star\|) \Bigr)(1+ \| \theta -\vartheta_\star \|^2)^{(1+p)/2} 
     + C_7 (1+ \|\vartheta_\star\|).
    \end{align*}
By using the above inequality, we obtain the following:
\begin{align}
    \nonumber \sup_{\theta \in \mathbb{R}^d, \|\theta- \vartheta_\star\| > 1 }& \frac{ \|P(\theta, \cdot) - \hat{P}(\theta, \cdot)\|_{V_p} }{\hat{V}_p(\theta)}  \\
    \leq& \sup_{\theta \in \mathbb{R}^d, \|\theta- \vartheta_\star\| > 1 } \frac{ \Bigl(C_7 + C_6(2+ \|\vartheta_\star\|) \Bigr)(1+ \| \theta -\vartheta_\star \|^2)^{(1+p)/2} 
     + C_7 (1+ \|\vartheta_\star\|)}{(1+\|\theta- \vartheta_\star\|^2)^{(1+p)/2}}\nonumber \\
     \leq& \Bigl(C_7 + C_6(2+ \|\vartheta_\star\|) \Bigr) + C_7 (1+ \|\vartheta_\star\|)\nonumber\\
     = & (C_6 + C_7)(2+ \|\vartheta_\star\|).
\end{align}

By combining all these estimates, we obtain the following inequality:
\begin{align*}
    \sup_{\theta \in \mathbb{R}^d } \frac{ \|P(\theta, \cdot) - \hat{P}(\theta, \cdot)\|_{V_p} }{\hat{V}_p(\theta)}  \leq (C_6 + C_7)(2+ \|\vartheta_\star\|),
\end{align*}
where it follows from \eqref{defn:C:6}, \eqref{defn:C:4:C:5}, \eqref{defn:C:1}, \eqref{defn:C:2:C:3}, \eqref{defn:C:2} that
\begin{align*}
C_{6}&= C_4 \left(1-2\eta m + \eta^2 K_1^2\right)^{p/2}
\\
&=2C_{1}C_{2}\left(1-2\eta m + \eta^2 K_1^2\right)^{p/2}
\\
&=\frac{8K_{2}D\eta}{n\sigma}\frac{\Gamma(1+\frac{1}{\alpha})}{\sqrt{\pi}}\left(1-2\eta m + \eta^2 K_1^2\right)^{p/2},
\end{align*}
and it follows from that \eqref{defn:C:7}, \eqref{defn:C:4:C:5}, \eqref{defn:C:p}, \eqref{defn:C:1}, \eqref{defn:C:2:C:3}, \eqref{defn:C:2}, \eqref{defn:C:3} that
\begin{align*}
C_{7}&= C_4 \left(\sqrt{2}  + (2\eta K)^{p/2}\right) + C_5 \sqrt{2} \sigma^{p} \sqrt{C(p)}
\\
&= 2C_{1}C_{2} \left(\sqrt{2}  + (2\eta K)^{p/2}\right) + 2C_{1}C_{3} \sqrt{2} \sigma^{p} \sqrt{2^{p}\frac{\Gamma(p+\frac{d}{2})}{\Gamma(\frac{d}{2})}}
\\
&= \frac{4K_{2}D\eta}{n\sigma}\left(\frac{2}{\sqrt{\pi}}\Gamma\left(1+\frac{1}{\alpha}\right) \left(\sqrt{2}  + (2\eta K)^{p/2}\right) + \frac{\Gamma(\frac{1+\alpha-p}{\alpha})}{\Gamma(\frac{3-p}{2})} \sqrt{2} \sigma^{p} 2^{p/2}\left(\frac{\Gamma(p+\frac{d}{2})}{\Gamma(\frac{d}{2})}\right)^{1/2}\right).
\end{align*}
This completes the proof.
\end{proof}

\subsection{Proof of Lemma~\ref{lem:erg}}

We first prove contraction for $p\leq 1$. 
\begin{lemma}
\label{lem:erg_new_p}
Let $P$ be the transition kernel associated with the Markov process $(\theta_k)_{k\geq 0}$ (i.e., \eqref{eqn:gd}) and $\hat{P}$ be the transition kernel associated with $(\hat{\theta}_k)_{k\geq 0}$. Suppose that Assumptions~\ref{assump:1} and \ref{assump:grad} hold, and the step-size is chosen as $\eta < \min\{m/K_1^2, 1/m\}$. 
Suppose that Assumption~\ref{asmp:interp} holds and set $V_p(\theta) := (1 + \|\theta - \vartheta_\star\|^2)^{p/2}$, where $\vartheta_\star$ is defined in Assumption~\ref{asmp:interp} and $p\in (0, 1]$.
   Then, the following inequalities hold:
   \begin{align*}
           \nonumber (PV_p)(\theta)  \leq& \beta_p  V_p(\theta) + H_p,\\
           \nonumber (\hat{P}V_p)(\theta)  \leq& \beta_p  V_p(\theta) + H_p,
    \end{align*}
    where
    \begin{align*}
        \beta_p :=& 1-\eta m p/2, \\
        H_p :=&  1 + (2\eta K)^{p/2}  + \sigma  2^p \frac{\Gamma\left(1-\frac{p}{\alpha}\right) \Gamma\left(\frac{d+p}{2}\right)}{\Gamma(1-\frac{p}{2}) \Gamma(\frac{d}{2})}.
    \end{align*}
\end{lemma}

\begin{proof}
    We follow the same proof strategy that was introduced by \citet[Proposition 1.7]{chen2022approximation}. We begin by estimating $(PV_p)(\theta)$ as follows: 
    \begin{align}
        \nonumber (PV_p)(\theta) =& \mathbb{E}[V_p(\theta_1) \>|\> \theta_0 = \theta] \\
        \nonumber \leq& \mathbb{E}\left[1+ \|\theta_1 - \vartheta_\star\|^p \>|\> \theta_0 = \theta \right] \\
        \nonumber =&  \mathbb{E}\left[1+ \|\theta - \eta \nabla F(\theta,X_n) + \sigma \xi_1 - \vartheta_\star\|^p\right] \\
        \label{eqn:proof_erg1_newp} \leq& 1 + \|\theta  - \vartheta_\star - \eta \nabla F(\theta,X_n) \|^p + \sigma^p \mathbb{E} \| \xi_1\|^p,
    \end{align}
     where we used the inequality that $(x+y)^{p}\leq x^{p}+y^{p}$ for any $x,y\geq 0$ since $p\leq 1$.
        Let us now focus on the second term in \eqref{eqn:proof_erg1_newp}. We can compute that:
    \begin{align}
        \nonumber \|\theta  - \vartheta_\star - \eta \nabla F(\theta,X_n) \|^2 =& \|\theta  - \vartheta_\star\|^2 - 2\eta \langle \theta  - \vartheta_\star,  \nabla F(\theta,X_n) - \nabla F(\vartheta_\star,X_n) \rangle \\
        \nonumber &\qquad \qquad \qquad \qquad \qquad \qquad + \eta^2 \|\nabla F(\theta,X_n) - \nabla F(\vartheta_\star,X_n)\|^2 \\
        \label{eqn:proof_erg2_newp}
        \leq & \left(1-2\eta m + \eta^2 K_1^2\right) \|\theta  - \vartheta_\star\|^2 + 2\eta K,
    \end{align}
    where in \eqref{eqn:proof_erg2_newp} we used Assumptions~\ref{assump:1} and \ref{assump:grad}. Using \eqref{eqn:proof_erg2_newp} in \eqref{eqn:proof_erg1_newp}, and the fact that $p/2<1$ such that $(x+y)^{p/2}\leq x^{p/2}+y^{p/2}$ for every $x,y\geq 0$, we obtain:
    \begin{align}
        \nonumber (PV_p)(\theta) \leq& 1 + \left(\left(1-2\eta m + \eta^2 K_1^2\right) \|\theta  - \vartheta_\star\|^2 + 2\eta K \right)^{p/2} + \sigma^p \mathbb{E} \| \xi_1\|^p \\
        \nonumber \leq& 1 + (1-2\eta m + \eta^2 K_1^2)^{p/2} \|\theta  - \vartheta_\star\|^p + (2\eta K)^{p/2}  + \sigma^p \mathbb{E} \| \xi_1\|^p  \\
        \label{eqn:proof_erg3_newp} \leq& 1+ (1-\eta m p/2) \|\theta - \vartheta_\star\|^p +  (2\eta K)^{p/2}  + \sigma^p \mathbb{E} \| \xi_1\|^p \nonumber\\
        \leq&  1+ (1-\eta m p/2) (1+\|\theta - \vartheta_\star\|^2)^{p/2} +  (2\eta K)^{p/2}  + \sigma^p \mathbb{E} \| \xi_1\|^p,
    \end{align}
    where \eqref{eqn:proof_erg3_newp} follows from the condition $\eta < \min\{m/K_1^2, 1/m\}$ and Bernoulli's inequality.  
    By using the fact that 
    $$\mathbb{E}[\|\xi_1\|^p] = 2^p \frac{\Gamma\left(1-\frac{p}{\alpha}\right) \Gamma\left(\frac{d+p}{2}\right)}{\Gamma(1-\frac{p}{2}) \Gamma(\frac{d}{2})}, $$ 
    (see 
    \citet[Lemma 4.2]{deng2019exact})
    we obtain:
    \begin{align*}
           \nonumber (PV_p)(\theta)  \leq&  (1-\eta m p/2) V_p(\theta) + 1 + (2\eta K)^{p/2}  + \sigma^p  2^p \frac{\Gamma\left(1-\frac{p}{\alpha}\right) \Gamma\left(\frac{d+p}{2}\right)}{\Gamma(1-\frac{p}{2}) \Gamma(\frac{d}{2})}.
    \end{align*}
The proof for $\hat{P}$ is identical. This concludes the proof. 
\end{proof}

\begin{proof}[Proof of Lemma~\ref{lem:erg}]
By Lemma~\ref{lem:erg_new_p}, we have that:
\begin{align*}
    \nonumber (PV_p)(\theta)  \leq&  (1-\eta m p/2) V_p(\theta) + 1 + (2\eta K)^{p/2}  + \sigma^p  2^p \frac{\Gamma\left(1-\frac{p}{\alpha}\right) \Gamma\left(\frac{d+p}{2}\right)}{\Gamma(1-\frac{p}{2}) \Gamma(\frac{d}{2})} \\
    =:& (1-\eta m p/2) V_p(\theta) +  C_0.
\end{align*}
      Defining $\lambda := 1- \eta mp /4 < 1$, we then have:
    \begin{align*}
         (PV)(\theta) \leq& \lambda V(\theta) +  C_0 - (\eta mp/4) (1+ \|\theta - \vartheta_\star\|^2)^{p/2} \\
        \leq& \lambda V(\theta) +  C_0 - (\eta mp/4) \|\theta - \vartheta_\star\|^p. 
    \end{align*}
   By defining 
    \begin{align*}
        A :=& \left\{ \theta \in \mathbb{R}^d : \|\theta -\vartheta_\star\| \leq \left( \frac{4C_0}{\eta m p} \right)^{1/p} \right\},
    \end{align*}
    we then obtain
    \begin{align*}
        (PV)(\theta) \leq \lambda V(\theta) + 
       C_0 \mathds{1}_A(\theta),
    \end{align*}
    where $\mathds{1}_A$ denotes the indicator function for the set $A$: $\mathds{1}_A(\theta) = 1$ if $\theta \in A$ and $\mathds{1}_A(\theta) = 0$, otherwise. 
    As $\lambda <1$ and $A$ is compact, the result follows from \citet[Appendix A]{lu2022central} and \citet[Theorem 6.3]{meyn1992stability}. This completes the proof.
\end{proof}

\subsection{Proof of Lemma~\ref{lem:erg_new_p2}}

We begin by providing an additional technical background that will be necessary in the proofs.

\paragraph{L\'{e}vy processes.}

L\'{e}vy processes are stochastic processes with independent and stationary increments.
Their successive displacements can be viewed as the continuous-time
analogue of random walks.
L\'{e}vy processes in general admit jumps
and have heavy tails which are appealing
in many applications; see e.g.\ \citet{Cont2004}.
L\'{e}vy processes include the Poisson process, the Brownian motion,
the Cauchy process, and more generally stable
processes; see e.g. \cite{bertoin1996,ST1994,applebaum2009levy}.
In particular, the \emph{rotationally invariant $\alpha$-stable L\'{e}vy process}, denoted by $\mathrm{L}_{t}^{\alpha}$ in $\mathbb{R}^d$ and is defined as follows.
\begin{itemize}[topsep=0pt,leftmargin=.11in]
\item $\mathrm{L}_0^\alpha=0$ almost surely;
\item For any $t_{0}<t_{1}<\cdots<t_{N}$, the increments $\mathrm{L}_{t_{n}}^\alpha-\mathrm{L}_{t_{n-1}}^\alpha$
are independent;
\item The difference $\mathrm{L}_{t}^{\alpha}-\mathrm{L}_{s}^{\alpha}$ and $\mathrm{L}_{t-s}^{\alpha}$
have the same distribution as $(t-s)^{1/\alpha}\xi_{1}$, where $\xi_{1}$ follows the rotationally invariant stable distribution \eqref{stable:dist} for any $t>s$, i.e. $\mathbb{E}[e^{\iu u^\top (\mathrm{L}_{t}^{\alpha}-\mathrm{L}_{s}^{\alpha})}]=\mathbb{E}[e^{\iu u^\top \mathrm{L}_{t-s}^{\alpha}}]=\mathbb{E}[e^{\iu u^\top(t-s)^{1/\alpha}\xi_1}]=e^{- (t-s)\|u\|^\alpha}$ for any $u\in\mathbb{R}^{d}$;
\item $\mathrm{L}_{t}^{\alpha}$ has stochastically continuous sample paths, i.e.
for any $\delta>0$ and $s\geq 0$, $\mathbb{P}(\|\mathrm{L}_{t}^{\alpha}-\mathrm{L}_{s}^{\alpha}\|>\delta)\rightarrow 0$
as $t\rightarrow s$.
\end{itemize}
When $\alpha=2$, $\mathrm{L}_{t}^{\alpha}=\sqrt{2}\mathrm{B}_{t}$, where $\mathrm{B}_{t}$ is the standard $d$-dimensional Brownian motion.

\paragraph{Fractional Laplacian.}  The fractional Laplacian operator, denoted by $(-\Delta)^{\alpha / 2} $, is the infinitesimal generator
of the rotationally invariant $\alpha$-stable L\'{e}vy process $\mathrm{L}_{t}^{\alpha}$
that is defined as a principal value (p.v.) integral: for any $f:\mathbb{R}^{d}\rightarrow\mathbb{R}$ that is $\mathcal{C}^{2}$, we have
\begin{equation}\label{fractional:Laplacian}
(-\Delta)^{\alpha/2}f(x)=C_{d,\alpha}\cdot\mathrm{p.v.}\int_{\mathbb{R}^{d}}(f(x+y)-f(x))\frac{\mathrm{d}y}{\Vert y\Vert^{\alpha+d}},    
\end{equation}
where
\begin{equation}\label{C:d:alpha}
C_{d,\alpha}:=\alpha 2^{\alpha-1}\pi^{-d/2}\frac{\Gamma(\frac{d+\alpha}{2})}{\Gamma(1-\frac{\alpha}{2})}.    
\end{equation}

\begin{lemma}
\label{lem:fraclap}
Let $\alpha \in (1,2)$, $p\in [1,\alpha)$, and $V_p :\mathbb{R}^d \to \mathbb{R}$ be defined as $V_p(\theta)=(1+ \|\theta - x\|^2)^{p/2}$ for some $x\in \mathbb{R}^d$. Then:
\begin{align*}
    \left|(-\Delta)^{\alpha / 2} V_p(\theta)\right| \leq {\sf C} \left( \frac{ p (\sqrt{d}+2) }{2-\alpha} +   \frac{p}{\alpha -1} \|\theta -x\|^{p-1} + \frac{1}{\alpha -p} \right),
\end{align*}
where $(-\Delta)^{\alpha/2}$ is the fractional Laplacian given in \eqref{fractional:Laplacian} and
\begin{align*}
    {\sf C} := \alpha 2 ^{\alpha}  \frac{\Gamma\left( \frac{d +\alpha}{2} \right)}{\Gamma \left(1-\frac{\alpha}{2} \right) \Gamma\left(\frac{d}{2}\right)}.
\end{align*}
\end{lemma}

\begin{proof}[Proof of Lemma~\ref{lem:fraclap}]
We use the same proof technique introduced in the proof of \citet[Equation (A.3)]{chen2022approximation}. We provide the proof here for completeness.

By using \citet[Equation (A.2)]{chen2022approximation}, we have that:
\begin{align*}
    (-\Delta)^{\alpha / 2} V_p(\theta) &= C_{d,\alpha}\int_{\|y\|< 1} \int_{0}^{1}\int_{0}^{r} \left\langle\nabla^{2}V_{p}(\theta+sy),yy^{\top}\right\rangle \,\mathrm{d} s\,\mathrm{d} r\,\frac{\mathrm{d} y}{\|y\|^{\alpha+d}} \\
&\qquad+ C_{d,\alpha}\int_{\|y\|\geq 1}\int_{0}^{1} \left\langle\nabla V_{p}(\theta+ry),y\right\rangle \,\mathrm{d} r\,\frac{\mathrm{d} y}{\|y\|^{\alpha+d}},
\end{align*}
where $C_{d,\alpha}$ is defined in \eqref{C:d:alpha} and for any two matrices $A,B\in\mathbb{R}^{d\times d}$, $\langle A,B\rangle:=\sum_{i,j=1}^{d}A_{ij}B_{ij}$. Hence, by Lemma~\ref{lem:v_bounds}, we have that:
\begin{align}
    \left|(-\Delta)^{\alpha / 2} V_p(\theta)\right| &\leq C_{d,\alpha} p \left(\sqrt{d}+2\right) \int_{\|y\|< 1}  \frac{\|y\|^2}{\|y\|^{\alpha+d}} \mathrm{d} y\nonumber \\
    &\qquad+ C_{d,\alpha} p \int_{\|y\|\geq 1}\int_{0}^{1}  \|\theta + ry -x\|^{p-1} \|y\| \frac{1}{\|y\|^{\alpha+d}}\mathrm{d} r\, \mathrm{d} y.\label{2:terms}
\end{align}
Let us first focus on the first term in \eqref{2:terms}. As $y \mapsto \|y\|^{2-\alpha-d}$ is radially symmetric, we have that:
\begin{align*}
    \int_{\|y\|< 1}  \frac{\|y\|^2}{\|y\|^{\alpha+d}} \mathrm{d} y \leq {\sf s}_d \int_0^1 r^{1-\alpha  } \mathrm{d}r = \frac{ {\sf s}_d}{2-\alpha},
\end{align*}
where 
\begin{equation}
{\sf s}_d :=\frac{2\pi^{d/2}}{\Gamma(d/2)}
\end{equation}
is the surface area of the unit sphere $\mathbb{S}^{d-1} := \{x\in \mathbb{R}^d; \> \|x\|=1\}$.

We now focus on the second term in \eqref{2:terms}. As $p-1 \leq 1$, we have that:
\begin{align*}
\int_{\|y\|\geq 1}\int_{0}^{1} & \|\theta + ry -x\|^{p-1}  \frac{1}{\|y\|^{\alpha+d-1}}\mathrm{d} r\, \mathrm{d} y \\
\leq& \int_{\|y\|\geq 1}\int_{0}^{1} \left( \|\theta -x\|^{p-1} + r^{p-1}\|y\|^{p-1} \right)  \frac{1}{\|y\|^{\alpha+d-1}}\mathrm{d} r\, \mathrm{d} y \\
=& \|\theta -x\|^{p-1} \int_{\|y\|\geq 1}   \frac{1}{\|y\|^{\alpha+d-1}} \mathrm{d} y + \int_{\|y\|\geq 1}  \frac{1}{\|y\|^{\alpha+d-p}}  \int_{0}^{1}  r^{p-1} \mathrm{d} r\, \mathrm{d} y \\
=& \|\theta -x\|^{p-1} \int_{\|y\|\geq 1}   \frac{1}{\|y\|^{\alpha+d-1}} \mathrm{d} y + \frac1{p} \int_{\|y\|\geq 1}  \frac{1}{\|y\|^{\alpha+d-p}}  \mathrm{d} y \\
=&   \frac{{\sf s}_d}{\alpha -1} \|\theta -x\|^{p-1} + \frac{{\sf s}_d}{p (\alpha -p)} .
\end{align*}
By combining all these estimates and using the fact that:
\begin{align*}
       {\sf C} := C_{d,\alpha}{\sf s}_d = \alpha 2 ^{\alpha}  \frac{\Gamma\left( \frac{d +\alpha}{2} \right)}{\Gamma\left(\frac{d}{2}\right)\Gamma \left(1-\frac{\alpha}{2} \right)},   
\end{align*}
we arrive at:
\begin{align*}
    \left|(-\Delta)^{\alpha / 2} V_p(\theta)\right| \leq&  {\sf C} \left( \frac{ p (\sqrt{d}+2) }{2-\alpha} +   \frac{p}{\alpha -1} \|\theta -x\|^{p-1} + \frac{1}{\alpha -p} \right) .
\end{align*}
This concludes the proof. 
\end{proof}

\begin{proof}[Proof of Lemma~\ref{lem:erg_new_p2}]
The proof is inspired by the proof technique introduced in \citet[Lemma 1.8]{chen2022approximation}. 
Recall that our goal is upper-bounding $(PV_p)(\theta) = \mathbb{E}[V_p(\theta_1) | \theta_0 = \theta]$ in terms of $V_p(\theta)$. Setting $\theta_0 = \theta$, and denoting $b(\theta) := -\nabla F(\theta,X_n)$ for notational simplicity, we start by decomposing $V_p(\theta_1)$ as follows:
\begin{align}
    \nonumber V_p(\theta_1) =& V_p(\theta + \eta b(\theta) + \sigma \xi_1) \\
   \nonumber =&  V_p(\theta + \eta b(\theta) ) +  V_p(\theta + \eta b(\theta) + \sigma \xi_1) -  V_p(\theta + \eta b(\theta) ) \\
    \nonumber=& V_p(\theta) + \int_0^\eta\left\langle\nabla V_p(\theta +r b(\theta)),  b(\theta) \right\rangle \mathrm{d} r\\
    \label{eqn:lyp_newbigp_proof1}
    &\qquad \qquad +V_p(\theta + \eta b(\theta) + \sigma \xi_1) -  V_p(\theta + \eta b(\theta) ).
\end{align}
 Recalling that $\nabla V_p(\theta)=p \left(1+\|\theta - \vartheta_\star\|^2\right)^{-\frac{2-p}{2}} (\theta -\vartheta_\star)$, and we have  
\begin{align} \label{eqn:lyp_newbigp_proof2}
\int_0^\eta\left\langle\nabla V_p\left(\theta + r 
b(\theta) \right), 
 b(\theta) \right\rangle \mathrm{d} r
 = \int_0^\eta \frac{ p \langle \theta -  \vartheta_\star, b(\theta) \rangle + p r \|b(\theta)\|^2 }{\left(1+\|\theta + r b(\theta) - \vartheta_\star\|^2\right)^{\frac{2-p}{2}}} \mathrm{d} r.
\end{align}
By using Assumption~\ref{assump:grad}, we further obtain that:
\begin{align*}
    \langle \theta -  \vartheta_\star, b(\theta) \rangle =& \langle \theta -  \vartheta_\star, b(\theta) - b(\vartheta_\star) \rangle\\
    = & - \langle \theta -  \vartheta_\star, \nabla F(\theta,X_n) -\nabla F(\vartheta_\star,X_n) \rangle\\
    \leq& -m \|\theta -\vartheta_\star\|^2 +K. 
\end{align*}
Similarly, by using Assumption~\ref{assump:1}, we have:
\begin{align*}
    \|b(\theta)\|^2 = \|\nabla F(\theta,X_n) - \nabla F(\vartheta_\star,X_n) \|^2
    \leq  K_1^2 \|\theta - \vartheta_\star \|^2.
\end{align*}
Hence, the numerator in \eqref{eqn:lyp_newbigp_proof2} can be bounded as follows:
\begin{align*}
    p \langle \theta -  \vartheta_\star, b(\theta) \rangle + p r \|b(\theta)\|^2 \leq&  (-mp +rpK_1^2) \|\theta -\vartheta_\star\|^2 + pK\\
    \leq& -\frac{mp}{2}\|\theta -\vartheta_\star\|^2 + pK,
\end{align*}
where in the last line we used the fact that $0\leq r \leq \eta \leq \frac{m}{2K_1^2}$.

Now let us focus on the denominator. By using a similar strategy, we have that:
\begin{align}
    \label{eqn:proof_v_large_new1}
    1+\|\theta + r b(\theta) - \vartheta_\star\|^2 =& 1 + \|\theta-\vartheta_\star\|^2 + 2r \langle \theta-\vartheta_\star, b(\theta)  \rangle + r^2 \|b(\theta)\|^2 \\
    \nonumber
    \leq& 1 + \|\theta-\vartheta_\star\|^2 + 2r(-m \|\theta -\vartheta_\star\|^2 +K) + r^2 K_1^2 \|\theta - \vartheta_\star \|^2 \\
    \nonumber
    =& 1+ (1 - 2rm + r^2K_1^2) \|\theta-\vartheta_\star\|^2 + 2rK\\
    \nonumber
    \leq&  \|\theta-\vartheta_\star\|^2 + 2rK + 1.
\end{align}
By combining all these estimates, and using the fact $p\in[1,\alpha)$ for $\alpha \in (1,2)$ such that $p<2$, we have:
\begin{align*}
     &\frac{ p \langle \theta -  \vartheta_\star, b(\theta) \rangle + p r \|b(\theta)\|^2 }{\left(1+\|\theta + r b(\theta) - \vartheta_\star\|^2\right)^{\frac{2-p}{2}}} 
     \\
     &\leq\frac{ -\frac{mp}{2} \|\theta -\vartheta_\star\|^2 + pK}{\left(1+\|\theta + r b(\theta) - \vartheta_\star\|^2\right)^{\frac{2-p}{2}}}
     \\
      &\leq -\frac{mp}{2}\frac{  \|\theta -\vartheta_\star\|^2 }{\left( \|\theta-\vartheta_\star\|^2 + 2rK + 1\right)^{\frac{2-p}{2}}} + pK
      \\
      &=-\frac{mp}{2}\frac{  \|\theta -\vartheta_\star\|^2+2rK+1 }{\left( \|\theta-\vartheta_\star\|^2 + 2rK + 1\right)^{\frac{2-p}{2}}}
      +\frac{mp}{2}\frac{ 2rK+1}{\left( \|\theta-\vartheta_\star\|^2 + 2rK + 1\right)^{\frac{2-p}{2}}}+ pK
      \\
      &\leq-\frac{mp}{2}\left( \|\theta-\vartheta_\star\|^2 + 2rK + 1\right)^{\frac{p}{2}}+\frac{mp}{2}(2rK+1)^{\frac{p}{2}}+pK
      \\
      &\leq-\frac{mp}{2} V_p(\theta)+\frac{mp}{2}(2rK+1)^{\frac{p}{2}} + pK\\
      &\leq -\frac{mp}{2} V_p(\theta)+\frac{mp K^{\frac{p}{2}}}{2^{1-\frac{p}{2}}}r^{\frac{p}{2}}  + p\left(\frac{m}{2} +K\right),
\end{align*}
where we used the inequality $(x+y)^{p/2}\leq x^{p/2}+y^{p/2}$
for any $x,y\geq 0$ since $p/2<1$.
By using this in \eqref{eqn:lyp_newbigp_proof2}, we obtain:
\begin{align}
  \nonumber  \int_0^\eta\left\langle\nabla V_p\left(\theta + r 
b(\theta) \right), 
 b(\theta) \right\rangle \mathrm{d} r 
 \leq& -\frac{mp\eta}{2} V_p(\theta)  + \eta  p\left(\frac{m}{2} +K\right) + \frac{mp K^{\frac{p}{2}}}{2^{1-\frac{p}{2}}} \frac{ 2 \eta^{1+p/2} }{p+2}   \\
  \label{eqn:lyp_newbigp_proof3}
 \leq&  -\frac{mp\eta}{2} V_p(\theta)  + \eta \left(p\left(\frac{m}{2} +K\right) + m(2K)^{p/2} \right). 
\end{align}
We now continue with analyzing the term $V_p(\theta + \eta b(\theta) + \sigma \xi_1) -  V_p(\theta + \eta b(\theta) )$. Let us first define the function $h : \mathbb{R}^d \to \mathbb{R}$:
\begin{align*}
    h(x) := V_p(\theta + \eta b(\theta) + x).
\end{align*}
Denoting the rotationally-invariant stable process in $\mathbb{R}^d$ by $\mathrm{L}_t^\alpha$, we have that
\begin{align*}
    h(\mathrm{L}_0^\alpha) &= V_p\left(\theta + \eta b(\theta)\right), \qquad \text{almost surely}, \\
    h(\mathrm{L}_{\sigma^\alpha}^\alpha) &\stackrel{\mathrm{d}}{=}  V_p\left(\theta + \eta b(\theta) + \sigma \xi_1\right), 
\end{align*}
where $\stackrel{\mathrm{d}}{=}$ denotes equality in distribution. Applying It\^{o}'s formula on $h$, we obtain the following identity:
\begin{align*}
\left|\mathbb{E}\left[V_p(\theta + \eta b(\theta) + \sigma \xi_1 )- V_p(\theta +  \eta b(\theta) )\right]\right| =&|\mathbb{E}[h(\mathrm{L}_{\sigma^\alpha}^\alpha) - h(\mathrm{L}_{0}^\alpha)]| \\
=& \left| \int_0^{\sigma^\alpha} \mathbb{E}\left[(-\Delta)^{\alpha / 2} h\left(\mathrm{L}^\alpha_r\right) \right] \mathrm{d} r \right|\\
=& \left| \int_0^{\sigma^\alpha} \mathbb{E}\left[\underbrace{(-\Delta)^{\alpha / 2} V_p\left(\theta +\eta b(\theta) +  \mathrm{L}^\alpha_r\right)}_{{\sf =:A_r}}\right] \mathrm{d} r \right|\\
\leq&  \int_0^{\sigma^\alpha} \mathbb{E}\left[ \left|{\sf A_r} \right| \right ] \mathrm{d} r, 
\end{align*}
where $(-\Delta)^{\alpha / 2} $ is the fractional Laplacian.
By Lemma~\ref{lem:fraclap}, we have that:
\begin{align*}
    |{\sf A}_r| \leq& {\sf C} \left( \frac{ p (\sqrt{d}+2) }{2-\alpha} +   \frac{p}{\alpha -1} \|\theta + \eta b(\theta) + \mathrm{L}^\alpha_r - \vartheta_\star\|^{p-1} + \frac{1}{\alpha -p} \right).
\end{align*}
We now upper-bound the second term by using similar arguments that we used in \eqref{eqn:proof_v_large_new1}, \eqref{eqn:proof_erg2_newp}, and \eqref{eqn:proof_erg3_newp}:
\begin{align*}
    \|\theta + \eta b(\theta) + \mathrm{L}^\alpha_r - \vartheta_\star\|^{p-1} \leq& \|\theta  - \vartheta_\star - \eta \nabla F(\theta,X_n) \|^{p-1}  + \|\mathrm{L}^\alpha_r \|^{p-1} \\
    \leq& \left(1-\eta m (p-1)/2\right) \|\theta  - \vartheta_\star\|^{p-1} + (2\eta K)^{(p-1)/2} + \|\mathrm{L}^\alpha_r \|^{p-1}\\
    \leq& \|\theta  - \vartheta_\star\|^{p-1} + (2\eta K)^{(p-1)/2} + \|\mathrm{L}^\alpha_r \|^{p-1}.
\end{align*}
By using the above inequality and rearranging the terms, we obtain:
\begin{align}
    \nonumber \int_0^{\sigma^\alpha}\mathbb{E}|{\sf A}_r| \mathrm{d} r \leq&  {\sf C} \sigma^{\alpha} \left( \frac{ p (\sqrt{d}+2) }{2-\alpha}  +   \frac{p}{\alpha -1}  \left[\|\theta  - \vartheta_\star\|^{p-1} + (2\eta K)^{(p-1)/2} \right]  + \frac{1}{\alpha -p}\right) \\
    \label{eqn:proof_lyp_bigp_new_2}
    &+ \frac{{\sf C} p}{\alpha -1}  \int_0^{\sigma^\alpha} 
     \mathbb{E}  \|\mathrm{L}^\alpha_r \|^{p-1} \mathrm{d} r.
\end{align}
We can further compute the last integral in \eqref{eqn:proof_lyp_bigp_new_2} as follows:
\begin{align*}
    \int_0^{\sigma^\alpha} 
     \mathbb{E}  \|\mathrm{L}^\alpha_r \|^{p-1} \mathrm{d} r =& \left(\int_0^{\sigma^\alpha} r^{(p-1)/\alpha} 
      \mathrm{d} r\right) \mathbb{E}  \|\xi_1 \|^{p-1} \\
      =& \frac{2^{p-1} \alpha \sigma^{\alpha +p-1}}{\alpha+p-1}   \frac{\Gamma\left(1-\frac{p-1}{\alpha}\right) \Gamma\left(\frac{d+p-1}{2}\right)}{\Gamma(1-\frac{p-1}{2}) \Gamma(\frac{d}{2})}.
\end{align*}
In the last line, we used \citet[Lemma 4.2]{deng2019exact}.

By using Young's inequality, we further obtain that:
\begin{align*}
    \frac{ {\sf C} p }{\alpha-1} \sigma^\alpha \|\theta -\vartheta_\star \|^{p-1} =& \eta \underbrace{
\left[ \frac{\frac{\sigma^\alpha}{\eta} {\sf C} p}{\alpha-1}  \left(\frac{4}{mp}\right)^{p-1} \right]}_{=: {\sf B}_1} {\underbrace{ \left[ \frac{mp }{4} \|\theta - \vartheta_\star\| \right]}_{=:{\sf B}_2}}^{p-1} \\
\leq& \eta \left( \frac{{\sf B}_1^p}{p} +\frac{{\sf B}_2^{\frac{p}{p-1}}}{p/(p-1)} \right) \\
=& \eta \frac{{\sf B}_1^p}{p} + \eta \frac{m(p-1)}{4} \|\theta -\vartheta_\star\|^p\\
\leq& \eta \frac{{\sf B}_1^p}{p} + \eta \frac{m(p-1)}{4} (1+\|\theta -\vartheta_\star\|^2)^{p/2}\\
=& \eta \frac{m(p-1)}{4} V_p(\theta) + \eta \frac{{\sf B}_1^p}{p} .
\end{align*}
Using the inequality above in \eqref{eqn:proof_lyp_bigp_new_2} yields:
\begin{align}
  \label{eqn:lyp_newbigp_proof4}
\int_0^{\sigma^\alpha}\mathbb{E}|{\sf A}_r| \mathrm{d} r \leq& \eta \frac{m(p-1)}{4} V_p(\theta) + {\sf C}_1,
\end{align}
where 
\begin{align*}
    {\sf C}_1 :=&  \eta \frac{{\sf B}_1^p}{p}  + {\sf C} \sigma^{\alpha} \left( \frac{ p (\sqrt{d}+2) }{2-\alpha}  +   \frac{p  (2\eta K)^{(p-1)/2} }{\alpha -1}   + \frac{1}{\alpha -p}\right) \\
    &\qquad \qquad + \frac{{\sf C} p}{\alpha -1}  \frac{2^{p-1} \alpha \sigma^{\alpha +p-1}}{\alpha+p-1}   \frac{\Gamma\left(1-\frac{p-1}{\alpha}\right) \Gamma\left(\frac{d+p-1}{2}\right)}{\Gamma(1-\frac{p-1}{2}) \Gamma(\frac{d}{2})}.
\end{align*}
By using \eqref{eqn:lyp_newbigp_proof3} and \eqref{eqn:lyp_newbigp_proof4} in \eqref{eqn:lyp_newbigp_proof1}
, we obtain:
\begin{align*}
    (PV_p)(\theta) &\leq \left(1 - \frac{mp\eta}{4}\right) V_p(\theta) +  \eta \left(p\left(\frac{m}{2} +K\right) + m(2K)^{p/2} \right) + {\sf C}_1\\
    &=: \beta_p V_p(\theta) + H_p.
\end{align*}
The proof of $(\hat{P}V_p)(\theta)$ is identical. This completes the proof.
\end{proof}

\subsection{Proof of Theorem~\ref{thm:dp_gd}}

\begin{proof}%
We will bound $\mathrm{TV}(\theta_k,\hat{\theta}_k)$ by using Lemma~\ref{lem:perturb_new} and the result will directly follow from Proposition~\ref{prop:tv_dp}. Let $P$ and $\hat{P}$ be the transition kernels associated with the Markov processes $(\theta_k)_{k\geq 0}$ and $(\hat{\theta}_k)_{k\geq 0}$, respectively. Furthermore assume that $\theta_0 = \hat{\theta}_0$ and denote $p_0$ as the common law of $\theta_0$ and $\hat{\theta}_0$.  

To invoke Lemma~\ref{lem:perturb_new}, we will use our intermediate results. More precisely, recalling that $p\in (0,\min(\frac1{2},\alpha-1))$, by Lemma~\ref{lem:erg}, there exist a Lyapunov function $V_p$, such that it holds that 
\begin{align}
    \left\|P^k(\theta, \cdot)-\pi\right\|_{V_p} \leq c V_p(\theta) \rho^k, \qquad\text{for any $\theta \in \mathbb{R}^d, k \in \mathbb{N}$},    
    \end{align}
for some $c>0$ and $\rho \in (0,1)$.

By Lemma~\ref{lem:erg_new_p2}, for $\hat{V}_p = V_{1+p}$, the following inequality holds:
    \begin{align*}
        (\hat{P}\hat{V}_p)(\theta)  \leq&  \beta_{1+p} \hat{V}_p(\theta) + H_{1+p},
    \end{align*}
where $\beta_{1+p}$ and $H_{1+p}$ are defined in Lemma~\ref{lem:erg_new_p2}.

Finally, by Lemma~\ref{lem:gamma_gd_new}, we have that
     \begin{align*}
        \gamma =& \sup_{\theta \in \mathbb{R}^d} \frac{\|P(\theta, \cdot) - \hat{P}(\theta, \cdot)\|_{V_p}}{\hat{V}_p(\theta)} 
        \leq \frac{ C_\gamma } {n } .
    \end{align*}
Now, we can invoke Lemma~\ref{lem:perturb_new}: for all $k$, we have that
\begin{align}
\label{eqn:proof_thm1}
\mathrm{TV}(\theta_k, \hat{\theta}_k) \leq \frac1{n} \cdot \frac{ c\left(1-\rho^k\right) C_\gamma H_{1+p}}{2(1-\rho)(1-\beta_{1+p})}.
\end{align}
This completes the proof. 
\end{proof}

\subsection{Proof of Lemma~\ref{lem:gamma_gd_withoutasmp}}

\begin{proof}
By using the same steps of the proof of Lemma~\ref{lem:gamma_gd_new} and using the exact same notation (we will deliberately avoid defining some of the constants that are already defined in the proof of Lemma~\ref{lem:gamma_gd_new} for brevity), we have that 
\begin{align}
     \|P(\theta, \cdot) - \hat{P}(\theta, \cdot)\|_{V_p} \leq \int_{\mathbb{R}_+} \sqrt{ 2\left\{\mu_{\theta,\phi} \left(V_p^2\right)+\hat{\mu}_{\theta,\phi}\left(V_p^2\right)\right\} \mathrm{KL}\left(\mu_{\theta,\phi} \mid \hat{\mu}_{\theta,\phi}\right)}  \>  p(\phi)  \mathrm{d}\phi,
     \label{eqn:proof_onestep_tmp_wout}
\end{align}
where
\begin{align*}
\left(\mathrm{KL}\left(\mu_{\theta,\phi} \mid \hat{\mu}_{\theta,\phi}\right)\right)^{1/2}   \leq& \phi^{-1/2} C_1 (1 + \|\theta - \hat{\theta}_\star\|) + \phi^{-1/2} C_1 \|\hat{\theta}_\star\|  \\
\leq&  \phi^{-1/2} C_1 (1 + \|\theta - \hat{\theta}_\star\|) + \phi^{-1/2} C_1 \frac{B+\sqrt{B^{2}+4mK}}{2m}, 
    \end{align*}
where the last line follows from Lemma~\ref{lem:nonconvex:minimizer}.

We now proceed to the estimation of the expectations of the  Lyapunov function under $\mu_{\theta,\phi}$ and $\hat{\mu}_{\theta,\phi}$. We start by the expecation with respect to $\mu_{\theta,\phi}$.
\begin{align}
    \nonumber \mu_{\theta,\phi}(V_p^2) =& \int_{\mathbb{R}^d} (1 + \|y - \theta_\star\|^2)^p \> p_\theta(y | \phi) \>  \mathrm{d}y \\
    \nonumber 
    \leq& 1 + \int_{\mathbb{R}^d}  \|y - \theta_\star\|^{2p} \> p_\theta(y | \phi) \>  \mathrm{d}y \\
    \nonumber =& 1 + \int_{\mathbb{R}^d} \|y - \theta_\star - (\theta - \eta \nabla F(\theta,X_n)) + (\theta - \eta \nabla F(\theta,X_n)) \|^{2p} \> p_\theta(y | \phi) \>  \mathrm{d}y \\
    \leq& 1 +  \|\theta - \theta_\star - \eta \nabla F(\theta,X_n) \|^{2p} + \int_{\mathbb{R}^d} \|y  - (\theta - \eta \nabla F(\theta,X_n)) \|^{2p} \> p_\theta(y | \phi) \>  \mathrm{d}y .\label{three:terms_noasmp}
\end{align}
Let us first focus on the second term in \eqref{three:terms_noasmp}. We can compute that:
    \begin{align}
        \nonumber \|\theta  - \vartheta_\star - \eta \nabla F(\theta,X_n) \|^2 =& \|\theta  - \theta_\star\|^2 - 2\eta \langle \theta  - \theta_\star,  \nabla F(\theta,X_n) - \nabla F(\theta_\star,X_n) \rangle \\
        \nonumber &\qquad \qquad \qquad \qquad \qquad \qquad + \eta^2 \|\nabla F(\theta,X_n) - \nabla F(\theta_\star,X_n)\|^2 \\
        \label{eqn:proof_onestepnew1_noasmp}
        \leq & \left(1-2\eta m + \eta^2 K_1^2\right) \|\theta  - \vartheta_\star\|^2 + 2\eta K,
    \end{align}
    where in \eqref{eqn:proof_onestepnew1_noasmp} we used Assumptions~\ref{assump:1} and \ref{assump:grad}.

For the third term in \eqref{three:terms_noasmp}, as before, we have that:
\begin{align*}
    \int_{\mathbb{R}^d} \|y  - (\theta - \eta \nabla F(\theta,X_n)) \|^{2p} \> p_\theta(y | \phi) \>  \mathrm{d}y 
    =& \phi^p \sigma^{2p} C(p).
\end{align*}

By combining these computations, we obtain that:
\begin{align}
\nonumber
     \mu_{\theta,\phi}(V_p^2) \leq& 1 +  \left(1-2\eta m + \eta^2 K_1^2\right)^p \|\theta  - \theta_\star\|^{2p} + (2\eta K)^p +  \phi^p \sigma^{2p} C(p) \\
     \nonumber
     \leq& 1 +  \left(1-2\eta m + \eta^2 K_1^2\right)^p \|\theta  - \hat{\theta}_\star\|^{2p} + (2\eta K)^p +  \phi^p \sigma^{2p} C(p) 
 \\
 \nonumber
 &+ \left(1-2\eta m + \eta^2 K_1^2\right)^p \|\theta_\star  - \hat{\theta}_\star\|^{2p}\\
 \nonumber
 \leq& 1 +  \left(1-2\eta m + \eta^2 K_1^2\right)^p \|\theta  - \hat{\theta}_\star\|^{2p} + (2\eta K)^p +  \phi^p \sigma^{2p} C(p) 
 \\
 \nonumber
 &+ \left(1-2\eta m + \eta^2 K_1^2\right)^p  \left(\|\theta_\star\|^{2p}  + \| \hat{\theta}_\star\|^{2p}\right) \\
 \nonumber
  \leq& 1 +  \left(1-2\eta m + \eta^2 K_1^2\right)^p \|\theta  - \hat{\theta}_\star\|^{2p} + (2\eta K)^p +  \phi^p \sigma^{2p} C(p) 
 \\
 &+ 2 \left(1-2\eta m + \eta^2 K_1^2\right)^p  \left(\frac{B+\sqrt{B^{2}+4mK}}{2m}\right)^{2p},
 \label{eqn:proof_noasmp2}
\end{align}
where in the last inequality, we again used Lemma~\ref{lem:nonconvex:minimizer}. 

We now proceed to estimating $\hat{\mu}_{\theta,\phi}(V_p^2)$.
\begin{align}
    \nonumber \hat{\mu}_{\theta,\phi}(V_p^2) =& \int_{\mathbb{R}^d} (1 + \|y - \theta_\star\|^2)^p \> \hat{p}_\theta(y | \phi) \>  \mathrm{d}y \\
    \nonumber
    \leq& 1 + \int_{\mathbb{R}^d}  \|y - \theta_\star\|^{2p} \> \hat{p}_\theta(y | \phi) \>  \mathrm{d}y \\
    \nonumber
    \leq&  1 + \int_{\mathbb{R}^d}  \|y - \hat{\theta}_\star\|^{2p} \> \hat{p}_\theta(y | \phi) \>  \mathrm{d}y   + \int_{\mathbb{R}^d}  \|\theta_\star - \hat{\theta}_\star\|^{2p} \> \hat{p}_\theta(y | \phi) \>  \mathrm{d}y \\
    \nonumber
    \leq& 1 + \int_{\mathbb{R}^d}  \|y - \hat{\theta}_\star\|^{2p} \> \hat{p}_\theta(y | \phi) \>  \mathrm{d}y   + 2^{2p}\left(\frac{B+\sqrt{B^{2}+4mK}}{2m}\right)^{2p} \\
    \nonumber
    \leq& 1 +  \|\theta - \hat{\theta}_\star - \eta \nabla F(\theta,\hat{X}_n)) \|^{2p} + \int_{\mathbb{R}^d} \|y  - (\theta - \eta \nabla F(\theta,\hat{X}_n)) \|^{2p} \> \hat{p}_\theta(y | \phi) \>  \mathrm{d}y \\
    & +
    2^{2p} \left(\frac{B+\sqrt{B^{2}+4mK}}{2m}\right)^{2p} .\label{three:terms_noasmp2}
\end{align}

By following similar steps, we also obtain:
\begin{align} 
    \hat{\mu}_{\theta,\phi}(V_p^2) 
    &\leq 1 +  \left(1-2\eta m + \eta^2 K_1^2\right)^p \|\theta  - \hat{\theta}_\star\|^{2p}\nonumber
    \\
    &\qquad\qquad\qquad+ (2\eta K)^p +  \phi^p \sigma^{2p} C(p) +
    2^{2p}\left(\frac{B+\sqrt{B^{2}+4mK}}{2m}\right)^{2p}.\label{eqn:proof_noasmp_1}
\end{align}
By combining and \eqref{eqn:proof_noasmp2} and \eqref{eqn:proof_noasmp_1}, we obtain:
\begin{align*}
    &\mu_{\theta,\phi} \left(V_p^2\right)+\hat{\mu}_{\theta,\phi}\left(V_p^2 \right) 
    \\
    &\leq 2 \Biggl(  1 +  \left(1-2\eta m + \eta^2 K_1^2\right)^p \|\theta  - \hat{\theta}_\star\|^{2p} + (2\eta K)^p +  \phi^p \sigma^{2p} C(p) + 2\hat{C }^{2p} \Biggr),
\end{align*}
where
\begin{align*}
    \hat{C}:=  \frac{B+\sqrt{B^{2}+4mK}}{2m}.
\end{align*}

Using these estimates for the integrand in \eqref{eqn:proof_onestep_tmp}, we get:
\begin{align*}
     &\sqrt{ 2\left\{\mu_{\theta,\phi} \left(V_p^2\right)+\hat{\mu}_{\theta,\phi}\left(V_p^2\right)\right\} \mathrm{KL}(\mu_{\theta,\phi} \mid \hat{\mu}_{\theta,\phi})} 
     \\
     &\leq 2 \left\{ 1 +  \left(1-2\eta m + \eta^2 K_1^2\right)^p \|\theta  - \hat{\theta}_\star\|^{2p} + (2\eta K)^p +  \phi^p \sigma^{2p} C(p) + 2\hat{C}^{2p}  \right\}^{1 / 2} \\
     & \qquad\qquad \cdot\left(\phi^{-1/2} C_1 (1 + \|\theta - \hat{\theta}_\star\|) +\phi^{-1/2} C_1 \hat{C} \right)\\
     &\leq 2\phi^{-1/2}  C_1 \left\{ 1 +  \left(1-2\eta m + \eta^2 K_1^2\right)^{p/2} \|\theta  - \hat{\theta}_\star\|^{p} + (2\eta K)^{p/2} +  \phi^{p/2} \sigma^{p} \sqrt{C(p)} + \sqrt{2}\hat{C}^p  \right\} \\
     & \qquad\qquad \cdot\left\{  (1 + \|\theta - \hat{\theta}_\star\|) + \hat{C} \right\}.
\end{align*}

By using above inequality in \eqref{eqn:proof_onestep_tmp_wout}, we obtain:
\begin{align}
  \nonumber &\|P(\theta, \cdot) - \hat{P}(\theta, \cdot)\|_{V_p}
  \\
  &\leq  \left(C_4  +  C_4 \left(1-2\eta m + \eta^2 K_1^2\right)^{p/2} \|\theta  - \hat{\theta}_\star\|^{p} + C_4 (2\eta K)^{p/2} + C_5 \sigma^{p} \sqrt{C(p)} + \sqrt{2} C_4 \hat{C}^p\right)\nonumber   \\
     & \qquad\qquad \cdot\left\{  (1 + \|\theta - \hat{\theta}_\star\|) + \hat{C} \right\}\nonumber \\
   &= \Bigl( \hat{C}_6 \| \theta -\hat{\theta}_\star \|^p + \hat{C}_7 \Bigr) \Bigl(  (1 + \|\theta - \hat{\theta}_\star\|) + \hat{C} \Bigr),
\end{align}
where
\begin{align}
    \hat{C}_6 :=& C_4 \left(1-2\eta m + \eta^2 K_1^2\right)^{p/2},\label{defn:hat:C:6} \\
    \hat{C}_7:= & C_4   + C_4 (2\eta K)^{p/2} + C_5 \sqrt{2} \sigma^{p} \sqrt{C(p) }+ \sqrt{2}C_4 \hat{C}^p. \label{defn:hat:C:7}
\end{align}
The rest of the proof follows the same lines of the proof of Lemma~\ref{lem:gamma_gd_new}, where we use $\hat{C}_6$ and $\hat{C}_7$ in place of $C_6$ and $C_7$.
This completes the proof.
\end{proof}

\section{Proofs of the Results of Section~\ref{sec:sgd}}\label{sec:app:proof_sgd}

\subsection{$V$-Uniform ergodicty}

\begin{lemma}
\label{lem:erg_sgd}
Let $P$ be the transition kernel associated with the Markov process \eqref{eqn:sgd}. Suppose that Assumptions~\ref{assump:1} and \ref{assump:grad} hold, and assume that the step-size is chosen as $\eta<\min\{m/K_1^2, 1/m\}$. Further assume that Assumption~\ref{asmp:interp} holds and set $V_p(\theta) = (1 + \|\theta - \vartheta_\star \|^2)^{p/2}$, where $\vartheta_\star$ is defined in Assumption~\ref{asmp:interp} and $p\in (0,1]$.  Then, for all $b \in \{1,\dots, n\}$
the process \eqref{eqn:sgd} admits a unique invariant measure $\pi$ such that the following inequality holds for some constants $C >0$ and $\rho \in (0,1)$:
    \begin{align}
    \label{eqn:lem:sgd_erg}
    \left\|P^k(\theta, \cdot)-\pi\right\|_{V_p} \leq C V_p(\theta) \rho^k,    
    \end{align}
    for all $\theta \in \mathbb{R}^d$ and $k \in \mathbb{N} $.

\end{lemma}

\begin{proof}
Recall that we define $V_p(\theta) = (1 + \|\theta - \vartheta_\star\|^2)^{1/2}$ for $p \in (0,1]$ in this part where $\vartheta_\star$ is defined in Assumption~\ref{asmp:interp}. We begin by estimating $(PV_p)(\theta)$ as follows:
    \begin{align}
        \nonumber (PV_p)(\theta) =& \mathbb{E}[V_p(\theta_1) \> | \> \theta_0 =\theta] \\
        \nonumber \leq& \mathbb{E}\left[1+ \|\theta_1 - \vartheta_\star\|^p\right] \\
        \nonumber =&  \mathbb{E}\left[1+ \|\theta - \eta \nabla F_1(\theta,X_n) + \sigma \xi_1 - \vartheta_\star\|^p\right] \\
        \label{eqn:proof_sgd_erg_new_1} \leq& 1 + \mathbb{E}\|\theta  - \vartheta_\star - \eta \nabla F_1(\theta,X_n) \|^p + \sigma^p \mathbb{E} \| \xi_1\|^p.
    \end{align}
    Let us now focus on the second term in \eqref{eqn:proof_sgd_erg_new_1}. It holds that:
    \begin{align}
        \nonumber \|\theta  - \vartheta_\star - \eta \nabla F_1(\theta,X_n) \|^2 =& \|\theta  - \vartheta_\star\|^2 - 2\eta \langle \theta  - \vartheta_\star,  \nabla F_1(\theta,X_n) - \nabla F_1(\vartheta_\star,X_n) \rangle \\
        \nonumber &\qquad \qquad \qquad \qquad \qquad + \eta^2 \|\nabla F_1(\theta,X_n) - \nabla F_1(\vartheta_\star,X_n)\|^2 \\
        \label{eqn:proof_sgd_erg_new_2}
        \leq & \left(1-2\eta m + \eta^2 K_1^2\right) \|\theta  - \vartheta_\star\|^2 + 2\eta K,
    \end{align}
    where we used Assumptions~\ref{assump:1} and \ref{assump:grad} to obtain \eqref{eqn:proof_sgd_erg_new_2}. 
    The result then follows by using the same arguments of the proof of Lemma~\ref{lem:erg}. This completes the proof. 
\end{proof}

\subsection{Estimation of the Lyapunov function}

\begin{lemma}
\label{lem:lyapunov_sgd}
Let $P$ be the transition kernel associated with the Markov process $(\theta_k)_{k\geq 0}$ (i.e., \eqref{eqn:sgd}) and $\hat{P}$ be the transition kernel associated with $(\hat{\theta}_k)_{k\geq 0}$ (i.e., \eqref{eqn:sgd2}). Suppose that Assumptions~\ref{assump:1}, \ref{assump:grad}, and the step-size satisfies: $\eta<\min\{m/K_1^2, 1/m\}$. Further assume that Assumption~\ref{asmp:interp} holds and set $V_p(\theta) = (1 + \|\theta - {\vartheta}_\star\|^2)^{p/2}$ for $p \in [1,\alpha)$. 
Then, the following inequalities hold:
   \begin{align*}
            (PV_p)(\theta)  \leq \beta_p  V_p(\theta) + H_p, \qquad 
            (\hat{P}V_p)(\theta)  \leq \beta_p  V_p(\theta) + H_p,
    \end{align*}
    where $\beta_p$ and $H_p$ are defined in Lemma~\ref{lem:erg_new_p2}.

\end{lemma}

\begin{proof}
Thanks to Assumption~\ref{asmp:interp}, $\vartheta_\star$ is also a stable point of $F_1(\cdot,X_n)$. Hence,  
by defining $b(\theta) := - \nabla F_1(\theta,X_n)$ (instead of $b(\theta) = \nabla F(\theta,X_n)$) and using the same arguments that we used in Lemma~\ref{lem:erg_new_p2}, we obtain the desired inequalities.
    This completes the proof. 
\end{proof}

\subsection{Distance between one-step transition kernels}

\begin{lemma}
    \label{lem:tv_sgd}

     Let $\alpha \in (1,2)$, $p\in (0,\min(\frac1{2},\alpha-1))$, $P$ be the transition kernel associated with the Markov process $(\theta_k)_{k\geq 0}$ in \eqref{eqn:sgd} and $\hat{P}$ be the transition kernel associated with $(\hat{\theta}_k)_{k\geq 0}$ in \eqref{eqn:sgd2}. Suppose that Assumptions~\ref{assump:1} and \ref{assump:grad} hold and further assume that $\sup_{x,\hat{x} \in \mathcal{X}} \|x-\hat{x}\|\leq D$, for some $D<\infty$. Suppose that Assumption~\ref{asmp:interp} holds and set $V_p(\theta)$ and $  \hat{V}_p(\theta)$ as in \eqref{eqn:v_vhat_new} and \eqref{eqn:v_vhat_new2}, respectively.
      Then, the following inequality holds for any $b \geq 1$:
        \begin{align*}
        \gamma =& \sup_{\theta \in \mathbb{R}^d} \frac{\|P(\theta, \cdot) - \hat{P}(\theta, \cdot)\|_{V_p}}{\hat{V}_p(\theta)} 
        \leq \frac{ C_\gamma } {n } ,
    \end{align*}
    where $C_\gamma>0$ is the same constant as in Lemma \ref{lem:gamma_gd_new}.
\end{lemma}

\begin{proof}%
We start by estimating the $V_p$-norm of the difference between the one-step transition kernels $P$ and $\hat{P}$. As opposed to the GD case (Lemma~\ref{lem:gamma_gd_new}), the Markov kernels have two sources of randomness: one coming from the injected noise $\xi_k$ and the random mini-batches $\Omega_k$. 
By conditioning on the random mini-batch $\Omega_1$ and using the same conditioning argument that we used in Lemma~\ref{lem:v_norm_cond}, for $\theta \in \mathbb{R}^d$, we have that:
\begin{align*}
    \left\|P(\theta, \cdot) - \hat{P}(\theta,\cdot)\right\|_{V_p} \leq \mathbb{E}_{\Omega_1} \left[ \left\|P(\theta, \cdot | \Omega_1) - \hat{P}(\theta,\cdot | \Omega_1)\right\|_{V_p} \right],
\end{align*}
where $P(\theta, \cdot | \Omega_1)$ and $\hat{P}(\theta,\cdot | \Omega_1)$ denote the transition probabilities when the random mini-batch is fixed to $\Omega_1$. Hence, the problem reduces to the estimation of $\|P(\theta, \cdot | \Omega_1) - \hat{P}(\theta,\cdot | \Omega_1)\|_{V_p}$. The rest of the proof is almost identical to the one of Lemma~\ref{lem:v_norm_cond}, where we replace $\nabla F$ with $\nabla F_1$. This concludes the proof.
\end{proof}

\section{Technical Lemmas}\label{sec:technical}

\begin{lemma}[{\citet[Lemma 24]{durmus2017nonasymptotic}}]
\label{lem:vnorm_kl}
    Let $\mu$ and $\nu$ be two probability measures on $\left(\mathbb{R}^d, \mathcal{B}\left(\mathbb{R}^d\right)\right)$ and $V: \mathbb{R}^d \rightarrow[1, \infty)$ be a measurable function. Then
$$
\|\mu-\nu\|_V \leq \sqrt{2}\left\{\nu\left(V^2\right)+\mu\left(V^2\right)\right\}^{1 / 2} \mathrm{KL}^{1 / 2}(\mu \mid \nu) .
$$
\end{lemma}

\begin{lemma}[{\citet[Lemma E.6]{zhu2023uniform}}]
\label{lem:nonconvex:minimizer}
Under Assumption~\ref{assump:grad}, 
we have
\begin{align*}
&\Vert\theta_{\ast}\Vert\leq\frac{B+\sqrt{B^{2}+4mK}}{2m}, \\
&\Vert\hat{\theta}_{\ast}\Vert\leq\frac{B+\sqrt{B^{2}+4mK}}{2m},
\end{align*}
where $\theta_*$ is a stable point of $\theta \mapsto F(\theta,X_n)$ and $\hat{\theta}_*$ is a stable point of $\theta \mapsto 
 F(\theta,\hat{X}_n)$. 
\end{lemma}

\begin{lemma}
    \label{lem:v_bounds}
    Let $\alpha \in (1,2)$, $p\in [1,\alpha)$, and $V_p :\mathbb{R}^d \to \mathbb{R}$ be defined as $V_p(\theta)=(1+ \|\theta - x\|^2)^{p/2}$ for some $x\in \mathbb{R}^d$. Then,
\begin{align*}
    \|\nabla V_p(\theta)\| \leq& p \|\theta -x\|^{p-1},\\
    \|\nabla^2 V_p(\theta)\|_\mathrm{F} \leq& p \left(\sqrt{d}+2\right).
\end{align*}
\end{lemma}

\begin{proof}[Proof of Lemma~\ref{lem:v_bounds}]
    By the definition of $V_p$, we have that:
    \begin{align}
        \nabla V_p(\theta)=&p \left(1+\|\theta - x\|^2\right)^{-\frac{2-p}{2}} (\theta -x), \label{grad:term}\\
        \nabla^2 V_p(\theta) =& p\frac{ \mathrm{I}_d}{\left(1+\|\theta - x\|^2\right)^{1-p/2}} + p(p-2) \frac{ (\theta -x)   (\theta-x)^\top }{\left(1+\|\theta - x\|^2\right)^{2-p/2}}.\label{Hessian:term}
    \end{align}
    First, we start with the gradient in \eqref{grad:term}. We can compute that
    \begin{align*}
        \|\nabla V_p(\theta) \| = p \frac{\|\theta -x\|}{\left(1+\|\theta - x\|^2\right)^{\frac{2-p}{2}}} 
        \leq p \frac{\|\theta -x\|}{\|\theta - x\|^{2-p}}
        =p \|\theta - x\|^{p-1}.
    \end{align*}
    Next, we proceed with the Hessian in \eqref{Hessian:term}. We can compute that
    \begin{align*}
        \|\nabla^2 V_p(\theta)\|_\mathrm{F} \leq& p \left\|\frac{ \mathrm{I}_d}{\left(1+\|\theta - x\|^2\right)^{1-p/2}}\right\|_\mathrm{F} + p(2-p) \left\| \frac{ (\theta -x)   (\theta-x)^\top }{\left(1+\|\theta - x\|^2\right)^{2-p/2}}\right\|_\mathrm{F} \\
        \leq& p \sqrt{d} + p(2-p) \frac{\|\theta - x\|^2}{\left(1+\|\theta - x\|^2\right)^{2-p/2}}.
    \end{align*}
    Simple calculation shows that the map $y \mapsto \frac{y}{(1+y)^{2-p/2}} $ with $y\geq 0$ is maximized when $y=\frac{2}{2-p}$. Hence we have that:
    \begin{align*}
        \|\nabla^2 V_p(\theta)\|_\mathrm{F} 
        \leq& p \sqrt{d} + 2p.
    \end{align*}
    This concludes the proof.
\end{proof}

\bibliography{heavy,levy,opt-ml}

\end{document}